\newcommand{\CLASSINPUTinnersidemargin}{0.5in} 
\newcommand{\CLASSINPUToutersidemargin}{0.5in} 
\newcommand{\CLASSINPUTtoptextmargin}{0.5in}   
\newcommand{\CLASSINPUTbottomtextmargin}{0.5in}

\documentclass[10pt,journal,compsoc]{IEEEtran}
\ifCLASSOPTIONcompsoc
  \usepackage{cite}
\else
  \usepackage{cite}  
\fi

\usepackage[l2tabu, orthodox]{nag}
\usepackage{mathtools,datetime}
\mathtoolsset{
showonlyrefs=true 
}
\usepackage{comment}
\usepackage[ruled,vlined]{algorithm2e}
\usepackage{amsmath, subfigure, dsfont, amsfonts, amssymb, graphicx, enumerate, xspace, cancel, verbatim, amsthm}
\usepackage[colorlinks,pdfpagelabels,citecolor=blue,plainpages=false]{hyperref}

\DeclareMathOperator*{\weight}{weight}
\DeclareMathOperator*{\argmin}{argmin}

\DeclareMathOperator*{\expec}{\mathbb E}

\newcommand{\grad}{\operatorname{\nabla}}

\newcommand{\prob}{{\mathbb P}}

\newcommand{\eod}{{${}$\\}}

\newcommand{\x}{{\mathbf x}}

\newcommand{\bu}{{\mathbf u}}

\newcommand{\bbf}{{\mathbf f}}

\newcommand{\bbb}{{\mathbf b}}

\newcommand{\bh}{{\mathbf h}}

\newcommand{\ba}{{\mathbf a}}

\newcommand{\y}{{\mathbf y}}

\newcommand{\bv}{{\mathbf{v}}}
\newcommand{\cP}{{\mathcal{P}}}

\newcommand{\bA}{{\mathbf{A}}}
\newcommand{\bI}{{\mathbf{I}}}

\newcommand{\bP}{{\mathbb{P}}}

\newcommand{\ccH}{{\mathcal{H}}}
\newcommand{\cC}{{\mathcal{C}}}

\newcommand{\indic}{{\mathds 1}}

\newcommand{\regret}{{\mathtt{Regret}}}
\newcommand{\proj}{{\mathtt{Proj}}}
\newcommand{\diam}{{\mathtt{diam}}}
\newcommand{\predg}{{\texttt{PS-Pred}}}
\newcommand{\nprop}{{\texttt{NProp}}}
\newcommand{\gradg}{{\texttt{PS-Grad}}}
\newcommand{\gregret}{{\mathtt{GRegret}}}
\newcommand{\pregret}{{\mathtt{PseudoRegret}}}

\newcommand*\bsigma{\ensuremath{\boldsymbol\sigma}}
\newcommand*\bvsigma{\ensuremath{\boldsymbol\varsigma}}

\newcommand*\bpi{\ensuremath{\boldsymbol\pi}}

\newcommand*\loss{\ensuremath{\boldsymbol\ell}}

\newcommand*\error{\mathcal E}

\newcommand{\subg}{\mathbf g}

\newcommand{\wt}{\mathbf{w}}
\newcommand{\Wt}{\mathbf{W}}
\newcommand{\vt}{\mathbf{v}}
\newcommand{\gain}{\mathbf{G}}

\newcommand{\bR}{{\mathbb R}}

\newcommand{\dd}{{\partial}}

\newcommand{\activ}{{\mathcal A}}

\newcommand{\fA}{{\mathfrak A}}

\newcommand{\cF}{{\mathcal F}}

\newtheorem{thm}{Theorem}
\newtheorem{cor}{Corollary}
\newtheorem{prop}{Proposition}
\newtheorem{lem}{Lemma}
\newtheorem{defn}{Definition}

\newtheorem{rem}{Remark}
\newtheorem{eg}{Example}
\newcommand{\captionfonts}{\normalsize}

\makeatletter  
\long\def\@makecaption#1#2{%
  \vskip\abovecaptionskip
  \sbox\@tempboxa{{\captionfonts #1: #2}}%
  \ifdim \wd\@tempboxa >\hsize
    {\captionfonts #1: #2\par}
  \else
    \hbox to\hsize{\hfil\box\@tempboxa\hfil}%
  \fi
  \vskip\belowcaptionskip}
\makeatother   

\begin{document}

\title{Deep Online Convex Optimization\\ with Gated Games}

\author{David Balduzzi\\
\texttt{\small david.balduzzi@vuw.ac.nz}
}

\markboth{}%
{Balduzzi}

\IEEEtitleabstractindextext{%

\begin{abstract}
  Methods from convex optimization are widely used as building blocks for deep learning algorithms. 
  However, the reasons for their empirical success are unclear, since modern convolutional networks (convnets), incorporating rectifier units and max-pooling, are neither smooth nor convex. Standard guarantees therefore do not apply.
  This paper provides the first convergence rates for gradient descent on rectifier convnets. The proof utilizes the particular structure of rectifier networks which consists in binary active/inactive gates applied on top of an underlying linear network. The approach generalizes to max-pooling, dropout and maxout. In other words, to precisely the neural networks that perform best empirically. 
  The key step is to introduce gated games, an extension of convex games with similar convergence properties that capture the gating function of rectifiers. The main result is that rectifier convnets converge to a critical point at a rate controlled by the gated-regret of the units in the network. 
  Corollaries of the main result include: (i) a game-theoretic description of the representations learned by a neural network; (ii) a logarithmic-regret algorithm for training neural nets; and (iii) a formal setting for analyzing conditional computation in neural nets that can be applied to recently developed models of attention.
\end{abstract}

}
\maketitle


\section{Introduction}
\label{sec:intro}

Deep learning algorithms have yielded impressive performance across a range of tasks, including object and voice recognition \cite{lecun:15}. The workhorse underlying deep learning is error backpropagation \cite{werbos:74,rumelhart:86,schmidhuber:15} -- a decades old algorithm that yields state-of-the-art performance on massive labeled datasets when combined with recent innovations such as rectifiers and dropout \cite{srivastava:14}.

Backprop is gradient descent plus the chain rule. Gradient descent has convergence guarantees in settings that are smooth or convex or both. However, \textbf{
\emph{modern convnets are neither smooth nor convex}}. Although it is well-known that convnets are not convex, it is perhaps under-emphasized that the spectacular recent results obtained by convnets on benchmarks such as ImageNet \cite{deng:09} rely on architectures that are not smooth. Starting with AlexNet in 2012, every winner of the ImageNet classification challenge has used rectifier (also known as rectilinear) activation functions \cite{krizhevsky:12,zeiler:14,szegedy:15,he:15}.

Rectifiers and max-pooling are non-smooth functions that are used in essentially all modern convnets \cite{jarrett:09,nair:10,glorot:11,maas:13,dahl:13,zeiler:13}. In fact, the representational power of rectifier nets derives precisely from their nondifferentiability: the number of nondifferentiable boundaries in the parameter space grows exponentially with depth \cite{pascanu:14}. It follows that none of the standard convergence guarantees from the optimization literature apply to modern convnets. 

In this paper, we provide the first convergence rates for convolutional networks with rectifiers and max-pooling. To do so we introduce a new class of \emph{gated games} which generalize the convex games studied by Stoltz and Lugosi in \cite{stoltz:07}. We reformulate learning in convnets as gated games and adapt results on convergence to correlated equilibria from convex to gated games.

\subsection{Open Questions in the Foundations of Deep Learning}

Theoretical questions about deep learning can be loosely grouped into four categories:
\begin{enumerate}[Q1.]
	\item \textbf{Representational power}
\end{enumerate}
The set of functions approximated by neural networks have been extensively studied. Early results show that neural networks with a single hidden layer are universal function approximators \cite{cybenko:89,leshno:93}. More recently, researchers have focused on the role of depth and rectifiers in function approximation \cite{eldan:15,telgarsky:16,daniely:16}.

\begin{enumerate}[Q2.]
	\item \textbf{Generalization guarantees}
\end{enumerate}
Standard guarantees from VC-theory apply to neural nets, although these are quite loose  \cite{vapnik:95,anthony:99}. Recent work by Hardt \emph{et al} shows that the convergence rate of stochastic gradient methods have implications for generalization bounds in both convex and nonconvex settings \cite{hardt:15}. Unfortunately, their results rely on a smoothness assumption\footnote{\label{ft:smooth}Function $f:\ccH\rightarrow\bR$ is smooth if $\forall\,\wt,\wt'\in\ccH$, $\exists\beta>0$ such that $\|\grad f(\wt)-\grad f(\wt')\|\leq \beta\|\wt-\wt'\|$. Rectifiers are not smooth for any $\beta$.} that does not hold for rectifiers or max-pooling. Thus, although suggestive, the results do not apply to modern convnets. Feng \emph{et al} have initiated a promising direction based on ensemble robustness is \cite{feng:16}, although robustness cannot be evaluated analytically.

A related problem is to better understand regularization methods such as dropout \cite{baldi:14,wager:13}. Regret-bounds for dropout have been found in the setting of prediction with expert advice \cite{vanerven:14}. However, it is unclear how to extend these results to neural nets.

\begin{enumerate}[Q3.]
	\item \textbf{Local and global minima}
\end{enumerate}
The third problem is to understand how far the critical points found by backpropagation are from local minima and the global minimum. The problem is challenging since neural networks are not convex. There has been theoretical work studying conditions under which gradient descent converges to local minima on nonconvex problems \cite{ge:15,lee:16}. The assumptions required for these results are quite strong, and include smoothness assumptions that do not hold for rectifiers. It has also been observed that saddles slow down training, even when the algorithm does not converge to a saddle point; designing algorithms that avoid saddles is an area of active research \cite{dauphin:14}.

Recent work by Choromanska \emph{et al} suggests that most local optima in neural nets have error rates that are reasonably close to the global optimum \cite{choromanska:14,choromanska:15}. Searching for good local optima may therefore be of less practical importance than ensuring rapid convergence.

\begin{enumerate}[Q4.]
	\item \textbf{Convergence rates}
\end{enumerate}
The last problem, and the focus of this paper, is to understand the convergence of gradient-based methods on neural networks. Speeding up the training time of neural nets is a problem of enormous practical importance. Although there is a large body of empirical work on optimizing neural nets, there are no theoretical guarantees that apply to the methods used to train rectifier convnets since they are neither smooth nor convex. 

Recent work has investigated the convergence of proximal methods for nonconvex nonsmooth problems \cite{bolte:14,li:15}. However computing $\text{prox}$- operators appears infeasible for neural nets. Interesting results have been derived for variance-reduced gradient optimization in the nonconvex setting \cite{allen-zhu:16,reddi:16,reddi:16a}, although smoothness is still required.

\subsection{Outline}

Training modern convnets, with rectifiers and max-pooling, entails searching over a rich subset of a universal class of function approximators with a loss function that is neither smooth nor convex. There is little hope of obtaining useful convergence results at this level of generality. It is therefore necessary to utilize the structure of rectifier networks.

Our strategy is to decompose neural nets into interacting optimizers that are easier to analyze individually than the net as a whole. In short, the strategy is to import techniques from game theory into deep learning.

\subsubsection{Utilizing network structure}
We make two observations about neural network structure. The first, section~\ref{sec:warmup}, is to reformulate linear networks as convex games where the players are the units in the network. Although the loss is not a convex function of the weights of the network; it is a convex function of the weights of the individual players. The observation connects the dynamics of games under no-regret learning to the dynamics of linear networks under backpropagation.

Linear networks are a special case. It is natural to ask whether neural networks with nonlinearities are also convex games. Unfortunately, the answer is no: introducing any nonlinearity breaks convexity.\footnote{In short, the ``nonlinearity'' $f$ would have to be affine since we need all linear combinations of $f$ to be convex, including $f(a)$ and $-f(a)$ $\forall a$.} Although the situation seems hopeless it turns out, remarkably, that game-theoretic convergence results can be imported -- despite nonconvexity -- for precisely the nonlinearities used in modern convnets.

The second observation, section~\ref{sec:relu}, is that a rectifier network is a linear network equipped with \emph{gates} that control whether units are active for a given input. If a unit is inactive during the feedforward sweep, it is also inactive during backpropagation, and therefore does not update its weights. This motivates generalizing convex games to gated games.

\subsubsection{Gated games}
In a classical game, each player chooses a series of actions and, on each round, incurs a convex loss. The \emph{regret} of a player is the difference between its cumulative loss and what the cumulative would have been had the player chosen the best action in hindsight \cite{cesa:06}. Players can be implemented with so-called no-regret algorithms that minimize their loss relative to the best action in hindsight. More precisely, a no-regret algorithm has sublinear cumulative regret. The regret per round therefore vanishes asymptotically. 

Section~\ref{sec:gt} introduces \emph{gated games} where players only incur a convex loss on rounds for which they are \emph{active}. After extending the definitions of regret and correlated equilibrium to gated games, proposition~\ref{prop:corr_eqm} shows that if players follow no-gated-regret strategies, then they converge to a correlated equilibrium. Gated players generalize the sleepy experts introduced by Blum in \cite{blum:97}, see also \cite{freund:97a}. 

A useful technical tool is \emph{path-sum games}. These are games constructed over directed acyclic graphs with weighted edges. Lemma~\ref{prop:struc} shows that path-sums encode the dynamics of the feedforward and feedback sweeps of rectifier nets. Proposition~\ref{prop:net_circ} shows that path-sum games are gated games and proposition~\ref{prop:conv} extends the result to convolutional networks.

\subsubsection{Summary of contribution}
The main contributions of the paper are as follows:
\begin{enumerate}[C1.]
	\item \textbf{Theorem~\ref{thm:critical}: Rectifier convnets converge to a critical point under backpropagation at a rate controlled by the \emph{gated-regret} of the units in the network.}
\end{enumerate}
Corollary~\ref{cor:ogd} specializes the result to gradient descent. To the best of our knowledge, there are no previous convergence rates applicable to neural nets with rectifier nonlinearities and max-pooling. Finding conditions that guarantee convergence to local minima is deferred to future work.

The results derive from a detailed analysis of the internal structure of rectifer nets and their updates under backpropagation. They require no new ideas regarding optimization in general. Our methods provide the first rigorous explanation for how methods designed for convex optimization improve convergence rates on modern convnets. The results do not apply to all neural networks: they hold for precisely the neural networks that perform best empirically \cite{krizhevsky:12,zeiler:14,szegedy:15,he:15,jarrett:09,nair:10,glorot:11,maas:13,dahl:13,zeiler:13}.

The philosophy underlying the paper is to decompose training neural nets into two distinct tasks: communication and optimization. Communication is handled by backpropagation which sends the correct gradient information to players (units) in the net. Optimization is handled locally by the individual players.  Note that although this philosophy is widely applied when designing and implementing neural nets, it has been under-utilized in the analysis of neural nets. The role of players in a convnet is encapsulated in the \emph{Gated Forecaster} setting, Section~\ref{sec:doco}. Our results provide a dictionary that translates the guarantees applicable to any no-regret algorithm into a convergence rate for the network as a whole. 

\begin{enumerate}[C2.]
	\item \textbf{Reformulate neural networks as games.}
\end{enumerate}
The primary conceptual contribution of the paper is to connect game theory to deep learning. An interesting consequence of the main result is corollary~\ref{cor:rep_learning} which provides a compact description of the weights learned by a neural network via the signal underlying correlated equilibrium. More generally, neural nets are a basic example of a game with a structured communication protocol (the path-sums) which determines how players interact \cite{balduzzi:gg}. It may be fruitful to investigate broader classes of structured games.

It has been suggested that rectifiers perform well because they are nonnegative homogeneous which has implications for regularization \cite{neyshabur:15} and robustness to changes in initialization \cite{daniely:16}. Our results provide a complementary explanation. Rectifiers simultaneously (i) introduce a nonlinearity into neural nets providing them with enormous representational power and (ii) act as gates that select subnetworks of an underlying linear neural network, so that convex methods are applicable with guarantees.

\begin{enumerate}[C3.]
	\item \textbf{Logarithmic regret algorithm.}
\end{enumerate}
As a concrete application of the gated forecaster framework, we adapt the Online Newton Step algorithm \cite{hazan:07a} to neural nets and show it has logarithmic-regret, corollary~\ref{cor:log}. The resulting algorithm approximates Newton's method locally at the level of individual units of the network -- rather than globally for the network as a whole. The local, unit-wise implementation reduces computational complexity and sidesteps the tendency of quasi-newton methods to approach saddle points.

\begin{enumerate}[C4.]
	\item \textbf{Conditional computation.}
\end{enumerate}
A secondary conceptual contribution is to introduce a framework for conditional computation. Up to this point, we assumed the gate is a fixed property of the game. Concretely, gates correspond to rectifiers and max-pooling in convolutional networks -- which are baked into the architecture before the network is exposed to data. It is natural to consider \emph{optimizing} when players in a gated game are active, section~\ref{sec:ls}. 

Recent work along these lines has applied reinforcement learning algorithms to find data-dependent dropout policies \cite{bacon:15,bengio:15e}. Conditional computation is closely related to models of attention \cite{mnih:14,gregor:15}. Slightly further afield, long-short term memories (LSTMs) and Gated Recurrent Units (GRUs) use complicated sets of sigmoid-gates to control activity within memory cells \cite{hochreiter:97,cho:14}. Unfortunately the resulting architectures are difficult to analyze; see \cite{balduzzi:16} for a principled simplification of recurrent neural network architectures motivated by similar considerations to the present paper. 
 
As a first step towards analyzing conditional computation in neural nets, we introduce the \emph{Conditional Gate (CoG)} setting. CoGs are contextual bandits or contextual partial monitors that optimize when sets of units are active. CoGs are a second class of players that can be introduced into in neural games, and may provide a useful tool when designing deep learning algorithms.

\subsubsection{Caveat}
Neural nets are typically trained on minibatches sampled i.i.d. from a dataset. In contrast, the analysis below provides guarantees in adversarial settings. Our results are therefore conservative. Extending the analysis to take advantage of stochastic settings is an important open problem. However, it is worth mentioning that neural nets are increasingly applied to data that is \emph{not} i.i.d. sampled. For example, adversarially trained generative networks have achieved impressive performance \cite{goodfellow:14,denton:15}. Similarly, there has been spectacular progress applying neural nets to reinforcement learning \cite{Mnih:2015wq,silver:16}.

Activity within a neural network is not i.i.d. even when the inputs are, a phenomenon known as internal covariate shift \cite{ioffe:15}. Two relevant developments are batch-normalization \cite{ioffe:15} and optimistic mirror descent \cite{daskalakis:14,rakhlin:13,syrgkanis:15}. Batch normalization significantly reduces the training time of neural nets by actively reducing internal covariate shift. Optimistic mirror descent takes advantage of the fact that all players in a game are implementing no-regret learning to speed up convergence.  It is interesting to investigate  whether reducing internal covariate shift can be understood game-theoretically, and whether optimistic learning algorithms can be adapted to neural networks.

\subsection{Related work}

A number of papers have brought techniques from convex optimization into the analysis of neural networks. A line of work initiated by Bengio in \cite{bengio:06a} shows that allowing the learning algorithm to choose the number of hidden units can convert neural network optimization in a convex problem, see also \cite{bach:14}. A convex multi-layer architecture is developed in \cite{aslan:13,aslan:14}. Although these methods are interesting, they have not achieved the practical success of convnets. In this paper, we analyze convnets \emph{as they are} rather than proposing a more tractable, but potentially less useful, model.

Game theory was developed to model interactions between humans \cite{vonneumann:44}. However, it may be more directly applicable as a toolbox for analyzing \emph{machina economicus} -- that is, interacting populations of algorithms that are optimizing objective functions \cite{parkes:15}. We go one step further, and develop a game-theoretic analysis of the internal structure of backpropagation. 

The idea of decomposing deep learning algorithms into cooperating modules dates back to at least the work of Bottou \cite{bottou:91}. A related line of work modeling biological neural networks from a game-theoretic perspective can be found in \cite{bb:12,balduzzi:13mv,balduzzi:14cpm,bvb:15}.

\section{Warmup: Linear Networks}
\label{sec:warmup}

The paper combines disparate techniques and notation from game theory, convex optimization and deep learning, and is therefore somewhat dense. To get oriented, we start with linear neural networks. Linear nets provide a simple but nontrivial worked example. They are not convex. Their energy landscapes and dynamics under backpropagation have been extensively studied and turn out to be surprisingly intricate \cite{baldi:88,saxe:14}.

\subsection{Neural networks}
\label{sec:nn}

Consider a neural network with $L-1$ hidden layers. Let $\bh_0:=\x$ denote the input to the network. For each layer $l\in\{1,\ldots,L\}$, set $\ba_l=\Wt_l\cdot \bh_{l-1}$ and $\bh_l= s(\ba_l)$ where $s$ is a (typically nonlinear) function applied coordinatewise. Let $\Wt:=\{\Wt_1,\ldots,\Wt_L\}$ denote the set of weight matrices. A convenient shorthand for the output of the network is $f_{\Wt}(\x)$.

For simplicity, suppose the output layer consists in a single unit and the output is a scalar (the assumption is dropped in section~\ref{sec:gt}). Let $(\x^{(i)},y^{(i)})_{i=1}^n$ denote a sample of labeled data and let $\ell(f,y)$ be a loss function that is convex in the first argument. Training the neural network reduces to solving the optimization problem 
\begin{equation}
  \label{eq:opt}
  \Wt^* = \argmin_\Wt \expec_{(\x,y)\sim \hat{\bP}}\Big[\ell \big(f_\Wt(\x),y\big)\Big],
\end{equation}
where $\hat{\bP}$ is the empirical distribution over the data. Training is typically performed using gradient descent
\begin{equation}
  \label{eq:gd}
  \Wt\leftarrow 
  \Wt - \expec_{(\x,y)\sim \hat{\bP}}\Big[\eta\cdot \grad_\Wt\ell \big(f_\Wt(\x),y\big)\Big].
\end{equation}
Since Eq.~\eqref{eq:opt} is not a convex function of $\Wt$, there are no guarantees that gradient descent will find the global minimum. 

\subsubsection{Backprop}
Let us recall how backprop is extracted from gradient descent. Subscripts now refer to units, not layers. Setting $\error:=\expec[\ell]$, Eq.~\eqref{eq:gd} can be written more concretely for a single weight as
\begin{equation}
  w_{ij} \leftarrow w_{ij} - \eta\cdot\frac{\dd\error}{\dd w_{ij}}.
\end{equation}
By the chain rule the derivative decomposes as $\frac{\dd\error}{\dd w_{ij}} = \frac{\dd \error}{\dd a_j}\frac{\dd a_j}{\dd w_{ij}}$ where $\delta_j = \frac{\dd \error}{\dd a_j}$ is the backpropagated error. Backprop computes $\delta_j$ recursively via
\begin{equation}
  \label{eq:bp}
  \underbrace{\frac{\dd \error}{\dd a_j}}_{\delta_j}
  = \underbrace{\sum_{\{k:j\rightarrow k\}}\overbrace{\frac{\dd \error}{\dd a_k}}^{\delta_k}\cdot \frac{\dd a_k}{\dd a_j}}_{\delta_j}
  = \underbrace{\sum_{\{k:j\rightarrow k\}}\delta_k\cdot w_{jk}h'_j}_{\delta_j}
\end{equation}

\subsection{Linear networks}
\label{sec:lin}

In a linear network, the function $s$ is the identity. It follows that the output of the network is
\begin{equation}
	f(\x) = f_{\Wt}(\x) 
  = \left(\prod_{l=1}^L\Wt_l\right)\cdot \x.
\end{equation}
For linear networks, the backpropagation computation in Eq.~\eqref{eq:bp} reduces to 
\begin{equation}
  \label{eq:bp_lin}
  \delta_j  
  = \sum_{\{k:j\rightarrow k\}}\delta_k\cdot w_{jk}
\end{equation}
It is convenient for our purposes to decompose $\delta_j$ slightly differently, by factorizing it into $\frac{\dd\error}{\dd f}$, the derivative of the loss with respect to the output, and $\frac{\dd f}{\dd a_j}$, the sensitivity of the network's output to unit $j$:
\begin{equation}
  \label{eq:bp_sens}
	\delta_j = \frac{\dd\error}{\dd a_{j}} 
	= \underbrace{\frac{\dd\error}{\dd f}\cdot \frac{\dd f}{\dd a_j}}_{\delta_j}
\end{equation}
We now reformulate the forward- and back- propagation in linear nets in terms of path-sums \cite{choromanska:14,bvb:15}:
\begin{defn}[path-sums in linear nets]\eod
  A \textbf{path} is a directed-sequence of edges connecting two units. Let $(i\leadsto j)$ denote the set of paths from unit $i$ to unit $j$. The \textbf{weight} of a path $\weight(\rho)$ is the product of the weights of the edges along the path.
  \begin{itemize}
    \item \textbf{sum(paths from $i$ to $j$):}
    \begin{equation}
      \sigma_{i \leadsto j}
    := \sum_{\rho \in (i\leadsto j)}\weight(\rho) 
    \end{equation}
    \item \textbf{sum(paths from $in$ to $j$):}
    \begin{equation}
      \sigma_{\bullet\leadsto j} 
     := \sum_{s\in in} x_s \cdot 
    \left(\sum_{\rho \in (s\leadsto j)}\weight(\rho)\right) 
    \end{equation}
    \item \textbf{sum(paths from $j$ to $out$):}
    \begin{equation}
      \sigma_{j\leadsto\bullet}
    := \sum_{\rho\in (j\leadsto \bullet)}\weight(\rho)
    \end{equation}
    \item \textbf{sum(paths avoiding $j$):}
    \begin{equation}
      \sigma_{- j} 
    := \sum_{s\in in}x_s 
    \cdot\left(\sum_{\left\{\rho |\rho \in (s\leadsto \bullet)\text{ and } j\not\in \rho\right\}}\weight(\rho)\right)
    \end{equation}
  \end{itemize}
\end{defn}

\begin{prop}[structure of linear nets]\label{prop:lin}\eod
  Let $\bsigma_{in(j)} = (\sigma_{\bullet\leadsto i})_{\{i:i\rightarrow j\}}$. 
  For a linear network as above,
  \begin{enumerate}
    \item \textbf{Feedforward computation of outputs.}\\
    The output of unit $j$ is $a_j = \sigma_{\bullet\leadsto j} = \langle\wt_j,\bsigma_{\text{in}(j)}\rangle$.
    
    \item \textbf{Sensitivity of network output.}\\
    The sensitivity of the network's output to unit $j$, denoted $\frac{\dd f_\Wt}{\dd a_j}$, is the sum of the weights of all paths from unit $j$ to the output unit:
    \begin{equation}
      \frac{\dd f_\Wt}{\dd a_j} = \sigma_{j\leadsto\bullet}
    \end{equation}

    \item  \textbf{Decomposition of network output.}\\
    The output of a linear network decomposes, with respect unit $j$, as
    \begin{equation}
      \label{eq:decomp}
      f_\Wt(\x) 
      = \langle\wt_j ,\bsigma_{in(j)}\rangle \cdot \sigma_{j\leadsto \bullet}
      + \sigma_{-j}    
      = a_j \cdot \sigma_{j\leadsto \bullet} + \sigma_{-j}    
    \end{equation}

    \item \textbf{Backpropagated errors.}\\
    Let $\beta := \frac{\dd \error}{\dd f}$ denote the derivative of $\error$ with respect to the output of the network. The backpropagated error signal received by unit $j$ is
    $\delta_j = \beta\cdot \sigma_{j\leadsto\text{out}}$.

    \item \textbf{Error gradients.}\\
    Finally,
      \begin{equation}
        \grad_{\wt_j}\error = \delta_j \cdot \bsigma_{in(j)} 
        = \beta\cdot \sigma_{j\leadsto \bullet}\cdot \bsigma_{in(j)}
      \end{equation}
  \end{enumerate}
\end{prop}

Note that $a_j$ is a linear function of the weight vector $\wt_j$ of unit $j$, and that neither the path-sums from $j$ nor the path-sums avoiding $j$ depend on $\wt_j$.

\begin{proof}
  Direct computations.
\end{proof}

The output of a linear neural network is a polynomial function of its weights. This can be seen from the path-sum perspective by noting that the output of a linear net is the sum over all paths from the input layer to the output unit.

\subsection{Game theory and online learning}

In this subsection we reformulate linear neural networks as convex games. 

\begin{defn}[convex game]\label{d:convex_game}\eod
  A \textbf{convex game} $([N],\ccH,\loss)$ consists of a set $[N]:=\{1,\ldots,N\}$ of players, actions $\ccH=\prod_{j=1}^N \ccH_j$ and loss vector $\loss=(\ell_1,\ldots, \ell_N)$. Player $j$ picks actions $\wt_i$ from a convex compact set $\ccH_j\subset \bR^{d_j}$. Player $j$'s loss $\ell_j:\ccH\rightarrow \bR$ is convex in the $j^\text{th}$ argument.

  The classical games of von Neumann and Morgensten \cite{vonneumann:44} are a special case of convex games where $\ccH_i = \triangle_{d_i}$ is the probability simplex over a finite set of actions $[d_i]$ available to each agent $i$, and the loss is multilinear.
\end{defn}

It is well known that, even in the linear case, the loss of a neural network is \emph{not} a convex function of its weights or the weights of its individual layers. However, the loss of a linear network is a convex function of the weights of the individual units.

\begin{prop}[linear networks are convex games]\label{prop:lin_conv_game}\eod
  The players correspond to the units, where we impose that weight vectors are chosen from compact, convexs set $\ccH_j\subset \bR^{n_j}$ for each unit $j$, where $n_j$ is the in-degree of unit $j$. 

  Let $\wt_{-j}$ denote the set of all weights in a neural network except those of unit $j$.
  Define the loss of unit $j$ as $\ell_j(\wt_j,\wt_{-j}, \x,y):= \ell(f_\Wt(\x),y)$. 
  Then $\ell_j$ is a convex function of $\wt_j$ for all $j$.  
\end{prop}

\begin{proof}
   Note that the loss of every unit is the same and corresponds to the loss of the network as a whole; the notation is introduced to emphasize the relevant parameters. By proposition~\ref{prop:lin}.3, the loss can be written
  \begin{equation}
    \ell_j(\wt_j,\wt_{-j},\x,y) = \ell\big(\langle\wt_j,\bsigma_{in(j)}\rangle\cdot \sigma_{j\leadsto \bullet} + \sigma_{-j},y\big),
  \end{equation}
  where $\bsigma_{in}(j)$, $\sigma_{j\leadsto \bullet}$ and $\sigma_{-j}$ are functions of $\wt_{-j}$ and $\x$, and so \emph{constant} with respect to $\wt_{j}$. It follows that the loss is the composite of an affine function of $\wt_j$ with a convex function, and so is convex. 
\end{proof}

\begin{rem}[any neural network is a game]\label{rem:general}\eod
  Any neural network can be reformulated as a game by treating the individual units as players. However, in general the loss will not be a convex function of the players' actions and so convergence guarantees are not available. The main conceptual contribution of the paper is to show that modern convnets form a class of games which, although not convex, are close enough that convergence results from game theory can be adapted to the setting.
\end{rem}

As a concrete example, consider a network equipped with the mean-square error. The loss of unit $j$ is 
  \begin{equation}
    \ell_j(\wt) 
    = \Big(\sigma_{j\leadsto \bullet}\cdot 
    \langle\wt_j,\bsigma_{in(j)}\rangle - \big(y - \sigma_{-j}\big)\Big)^2.
  \end{equation}
  Define the residue $y_{-j} := \frac{y - \sigma_{-j}}{\sigma_{j\leadsto \bullet}}$. Unit $j$'s loss can be rewritten
  \begin{equation}
    \ell_j(\wt) = \begin{cases}
      \sigma_{j\leadsto \bullet}^2\cdot
      \left(\langle\wt_j,\bsigma_{\text{in}(j)}\rangle-y_{-j}\right)^2 
      & \text{if }\sigma_{j\leadsto \bullet}\neq 0\\
      0 & \text{else.}
    \end{cases}    
    \end{equation}
  Thus, unit $j$ performs linear regression on the residue, amplified by a scalar that reflects the network's sensitivity to $j$'s output. 

The goal of each player $j$ in a game is to minimize its loss $\ell_j$. Unfortunately, this is not realistic, since $\ell_j(\wt_j,\wt_{-j})$ depends on the actions of other players. If the game is repeated, then an attainable goal is for players to minimize their regret. A player's cumulative \textbf{regret} is the difference between the loss incurred over a series of plays and the loss that would have been incurred had the player consistently chosen the best play in hindsight:
\begin{align}
  \label{e:regret}
  \regret_j(T) = \sup_{\wt_j\in\ccH_j}\frac{1}{T}\sum_{t=1}^T
  \Big(& \ell_j(\wt^t_1,\ldots, \wt^t_j,\ldots, \wt^t_N) \\
   - & \ell_j(\wt^t_1,\ldots,\wt_j,\ldots,\wt^t_N)\Big).
\end{align}
An algorithm has \textbf{no-regret} if $\regret_j(T)\xrightarrow[\infty]{T}0$. That is, an algorithm has no-regret (asymptotically) if $\regret_j(T)$ grows sublinearly in $T$. It is important to note that no-regret guarantees hold against \emph{any} sequence of actions by the other players in the game -- be they stochastic, adversarial, or something else. A player with no-regret plays optimally given the actions of the other players. Examples of no-regret algorithms on convex losses include online gradient descent, the exponential weights algorithm, follow the regularized leader, AdaGrad, and online mirror descent.

It was observed by Foster and Vohra \cite{foster:97} that, if players play according to no-regret online learning rules, then the average of the sequence of plays converges to a correlated equilibrium \cite{aumann:74}. Proposition~\ref{prop:corr_eqm} below shows a more general result: no-gated-regret algorithms converge to correlated equilibrium at a rate that depends on the gated-regret. 

Let us briefly recall the relevant notion of correlated equilibrium. A distribution $\bP\in\triangle_\ccH$ is an $\epsilon$-\textbf{coarse correlated equilibrium} if, for every player $j$, it holds that
\begin{equation}
  \label{eq:corr_eq}
  \expec_{\wt\sim \bP}\Big[\ell_j(\wt)\Big]  \leq \inf_{\wt_j\in\ccH_j}\expec_{\wt\sim \bP}\Big[\ell_j\big(\wt_j,\wt_{-j}\big)\Big] + \epsilon.
\end{equation}
When $\epsilon=0$ we refer to a coarse correlated equilibrium. The $\epsilon$-term in Eq.~\eqref{eq:corr_eq} quantifies the deviation of $\prob$ from a coarse correlated equilibrium \cite{blum:07}.  The notion of correlated equilibrium is weaker than Nash equilibrium. The set of correlated equilibria contains the convex hull of the set of Nash equilibria as a subset.

We thus have two perspectives on linear nets: as networks or as games. To train a network, we use algorithms such as gradient descent implemented via backpropagation. To play a game, the players use no-regret algorithms. Sections~\ref{sec:gt} and \ref{sec:doco} show the two perspectives are equivalent in the more general setting of modern convnets. In particular, correlated equilibria games map to critical points of energy landscapes. Our strategy is then to convert results about the convergence of convex games to correlated equilibria into results about the convergence of backpropagation on neural nets.

\section{Gated Games and Convolutional Networks}
\label{sec:gt}

This section presents a detailed analysis of rectifier nets. The key observation is that rectifiers act as gates, which leads directly to \emph{gated games}. Gates games are not convex. However, they are close enough that results on convergence to correlated equilibria can easily be adapted to the setting. 

The main technical work of the section is to introduce notation to handle the interaction between path-sums and gates. \emph{Path-sum games} are then introduced as a class of gated games capturing the dynamics of rectifier nets, see proposition~\ref{prop:net_circ}. Finally, we show how to extend the results to convnets.

\subsection{Rectifier networks}
\label{sec:relu}

Historically, neural networks typically used sigmoid $\sigma(a) = \frac{1}{1+e^{-a}}$ or tanh $\tau(a) = \frac{e^a-e^{-a}}{e^a+e^{-a}}$ nonlinearities. Alternatives were investigated by Jarrett \emph{et al} in \cite{jarrett:09}, who found that rectifiers $\rho(a) = \max(0,a)$ often perform much better than sigmoids in practice. Rectifiers are now the default nonlinearity in convnets \cite{nair:10,glorot:11,maas:13,dahl:13,zeiler:13}. There are many variants on the theme, including noisy rectifers, $\rho_N(a) = \max(0,a+N(0,\sigma(a))$, and leaky rectifers
  \begin{equation}
    \rho_L(a) = \begin{cases}
      a & \text{if }a>0\\
      0.01 a & \text{else}
    \end{cases}
  \end{equation}
introduced in \cite{nair:10} and \cite{maas:13} respectively.

\subsubsection{Rectifiers gate error backpropagation}
The rectifier is convex and differentiable except at $a=0$, with subgradient
\begin{equation}
  \indic(a) = \frac{d \rho}{da} = \begin{cases}
    1 & \text{if }a\geq0\\
    0 & \text{else.}
  \end{cases}
\end{equation}
The subgradient acts as an indicator function, which motivates the choice of notation. Substituting $\indic_j$ for $h_j'$ in the recursive backprop computation, Eq.~\eqref{eq:bp}, yields
\begin{equation}
  \label{eq:bp_relu}
  \delta_j = \frac{\dd\ell}{\dd a_j}   
  = \sum_{\{k:j\rightarrow k\}}\delta_k\cdot w_{jk}\cdot\indic_j
\end{equation}
Rectifiers act as \emph{gates}: the only difference between backprop in a linear network, Eq.~\eqref{eq:bp_lin}, and a rectifier network, Eq.~\eqref{eq:bp_relu}, is that some units are zeroed out during both the forward and backward sweeps. In the forward sweep, rectifiers zero out units which would have produced negative outputs; on the backward sweep, the rectifier subgradients zero out the exact same units by acting as indicator functions. 

Zeroed out (or \emph{inactive}) units do not contribute to the feedforward sweep and do not receive an error signal during backpropagation. In effect, the rectifiers select a \emph{linear subnetwork of active units} for use during forward and backpropagation.

\subsection{Gated Games}

We have seen that linear networks are convex games. Extending the result to rectifier networks requires generalizing convex games to the setting where only a subset of players are active on each round.

\begin{defn}[gated games]\label{d:gated_game}\eod
  Let $\cP(S)$ denote the powerset of $S$. A \textbf{gated game} $([N],\ccH,\loss,\activ)$ is a convex game equipped with a gate $\activ:\ccH\rightarrow \cP([N])$. Players $j\in\activ(\wt_{1:N})$ are \textbf{active}. Inactive players incur no loss. Each active player incurs a convex loss $\ell_j:\ccH_{\activ(\wt_{1:N})}\rightarrow \bR$, that depends on its action and the actions of the other active players; i.e. $\ccH_{\activ(\wt_{1:N})} := \prod_{j\in\activ(\wt_{1:N})}\ccH_j$. Inactive players do not incur a loss.
\end{defn}
The gated forecaster setting formalizes the perspective of a player in a gated game:
\renewcommand*{\algorithmcfname}{Setting}
\begin{algorithm}
  \caption{\textsc{Gated Forecaster}}\label{alg:sf}
  \DontPrintSemicolon
  \SetKwInOut{Input}{input}
  \Input{
    Set $E$ of other players\;
    initialize $\wt^1 \in \ccH$\;
  } 
  \For{rounds $t =1, 2, \ldots$}{
    Active players $E^t\subset E$ reveal actions $\x_E^t:=\{\x^t_i\}_{i\in E^t}$\;
    \eIf{\texttt{gate}$(\wt^t,\x^t)$}{
      Environment reveals convex $\ell^t$\; 
      Forecaster incurs loss $\ell^t\big(\wt^t,\x^t_E\big)$\;
      Forecaster updates weights $\wt^{t+1} \longleftarrow \texttt{weight}(\wt^t,\x^t_E,\ell^t)$
    }{
      Forecaster inactive; no loss incurred\; 
      Weights remain unchanged $\wt^{t+1}\longleftarrow \wt^t$
    }
  }
\end{algorithm}

In neural nets, rectifier functions, max-pooling, and dropout all act as forms of gates that control (deterministically or probabilistically) whether units actively respond to an input. Importantly, inactive units do not receive any error under backpropagation, as discussed in section~\ref{sec:relu}. 

In the gated setting, players only experience regret with respect to their actions when active. We therefore introduce \emph{gated-regret}:
\begin{align}
  \gregret(T)  = \frac{1}{T_i}&\left(\sum_{\{t\in[T]:i\in\activ(\wt^t)\}}
  \ell^t_i(\wt^t_i,\wt^t_{-i})\right.\\
  & - \left.\sup_{\wt_i\in\ccH_i}\;\sum_{\{t\in[T]:i\in\activ(\wt^t)\}} \ell^t(\wt_i,\wt^t_{-i})\right),
\end{align}
where $T_i:=\left|\{t\in[T]:i\in\activ(\wt^t)\}\right|$ is the number of rounds in which player $i$ is active.

\begin{rem}[permanently inactive units]\eod
  If a player is permanently inactive then, trivially, it has no gated-regret. This suggests there is a loophole in the definition that players can exploit. We make two comments. Firstly, players do not control when they are inactive. Rather, they optimize over the rounds they are exposed to. 

  Secondly, in practice, some units in rectifier networks do become inactive. The problem is mild: rectifier nets still outperform other architectures. Reducing the number of inactive units was one of the motivations for maxout units \cite{goodfellow:13}.
\end{rem}

The next step is to extend correlated equilibrium to gated games. The intuition behind correlated equilibrium is that a \textbf{signal} $\prob(\wt)$ is sent to all players which guides their behavior. However, inactive players do not observe the signal. The signal received by player $j$ \emph{when active} is the conditional distribution
\begin{equation}
  \label{eq:wake_signal}
  \bP_j := \bP\big(\wt|j\in\activ(\wt)\big) = 
  \begin{cases}
    \frac{\bP(\wt)}{\bP(\{\wt | j\in\activ(\wt)\})} & \text{if }j\in\activ(\wt)\\
    0 &\text{else.}
  \end{cases}
\end{equation}

The following proposition extends the result that no-regret learning leads to coarse correlated equilibria from convex to gated games:

\begin{prop}[no gated-regret $\rightarrow$ correlated equilibrium]\label{prop:corr_eqm}\eod
  Let $G=([N],\ccH,\loss,\activ)$ be a gated game and suppose that the players follow strategies with gated-regret $\leq \epsilon$. Then, the empirical distribution $\hat{\bP}$ of the actions played is an $\epsilon$-coarse correlated equilibrium.
\end{prop}

The rate at which the gated-regret of the players decay thus controls the rate of convergence to a correlated equilibrium.

\begin{proof}
  We adapt a theorem for two-player convex games by Hazan and Kale in \cite{hazan:07} to our setting. Since player $j$ has gated-regret $\leq \epsilon$, it follows that 
  \begin{align}
    \frac{1}{T_j}\sum_{\{t\in[T]:j\in\activ(\wt^t)\}}
    \Big(\ell_j(\wt^t) - \ell_j\big(\wt_j,\wt_{-j}^t\big)
    \Big)\leq \epsilon \quad\forall \wt_j\in\ccH_j.
  \end{align} 
  The empirical distribution $\hat{\bP}_j$ assigns probability $\frac{1}{T_j}$ to each joint action $\wt^t$ occurring while player $j$ is active. We can therefore rewrite the above inequality as
  \begin{equation}
    \expec_{\wt\sim \hat{\bP}_j}\Big[ \ell_j(\wt)\Big]
    -\expec_{\wt\sim \hat{\bP}_j}\Big[\ell_j\big(\wt_j,\wt_{-j}\big)\Big]
    \leq \epsilon
    \quad\quad \forall \wt_j\in\ccH_j
  \end{equation}
  and the result follows.
\end{proof}

\subsection{Path-Sum Games}

Let $G$ be a directed acyclic graph corresponding to a rectifier neural network with $N$ units that are not input units. We provide an alternate description of the dynamics of the feedforward and feedback sweep on the neural net in terms of path-sums. The definitions are somewhat technical; the underlying intuition can be found in the discussions of linear and rectifier networks above.

Let $d_j:=|\{i:i\rightarrow j\}|$ denote the indegree of node $j$. 
Every edge is assigned a \emph{weight}. In addition, each source node $s$ (with no incoming edges) is assigned a weight $w_s$. The weights assigned to source nodes are used to encode the input to the neural net. 

Recall that given path $\rho$, we write $\rho\in (j\leadsto k)$ if $\rho$ starts at node $j$ and finishes at node $k$. Given a set of nodes $A$, write $\rho\subset A$ if all the nodes along $\rho$ are elements of $A$. 

\begin{defn}[path-sums in rectifier nets]\label{def:path_sum}\eod
  The \textbf{weight of a path} is the product of the weights of the edges along the path. If a path starts at a source node $s$, then $w_s$ is included in the product. 

  Given a set of nodes $A$ and a node $j$, the \textbf{path-sum} $\sigma_{\bullet\leadsto j}^A$ is the sum of the weights of all paths in $A$ from source nodes to $j$:
  \begin{equation}
    \sigma_{\bullet\leadsto j}^A := \sum_{s\in\text{source}}\quad\sum_{\{\rho\,|\,\rho\subset A \text{ and }\rho\in(s\leadsto j)\}}\weight(\rho). 
  \end{equation}
  By convention, $\sigma_{\bullet\leadsto j}^A$ is zero if no such path exists (for example, if $j\not\in A$).

  The set $\activ$ of \textbf{active units} is defined inductively on $\kappa$, which tracks the length of the longest path from source units to a given unit: 
  \begin{equation}
    \text{Let }
    \activ_\kappa = \{\text{active units with longest path from a source }\leq \kappa\}.
  \end{equation}
  Source units are always active, so set $\activ_0 = \{\text{all source units}\}$.
  Suppose unit $j$ has source-path-length $\kappa+1$ and elements in $\activ_\kappa$ have been identified. Then, $j$ is active if it corresponds to
  \begin{itemize}
    \item a linear unit or
    \item a rectifier with $\sigma^{\activ_\kappa\cup\{j\}}_{\bullet\leadsto j}>0$.
   \end{itemize}
\end{defn}

For simplicity we suppress that $\activ$ is a function of weights from the notation. It is also convenient to drop the superscript $\activ$ via the shorthand $\varsigma_{\bullet\leadsto j} := \sigma_{\bullet\leadsto j}^\activ$. 

The following proposition connects active path-sums to the feedforward and feedback sweeps in a neural network:

\begin{prop}[structure of rectifier nets]\label{prop:struc}\eod
  Let $\wt_j= (w_i)_{\{i:i\rightarrow j\}}$ and $\bvsigma_{\text{in}(j)}:=(\varsigma_{\bullet\leadsto i})_{\{i:i\rightarrow j\}}$. 
  Further, introduce notation $\bvsigma_\text{out}=(\varsigma_o)_{o\in\text{out}}$ for the output layer. Then
  \begin{enumerate}[1.]
    \item \textbf{Feedforward outputs.}\\
    If inputs to the network are encoded in source weights as above then the output of unit $j$ in the neural network is $\varsigma_j$. Specifically, if $j$ is linear then $\varsigma_{\bullet\leadsto j} = \langle\wt_j,\bvsigma_{\text{in}(j)}\rangle$; if $j$ is a rectifier then $\varsigma_{\bullet\leadsto j} = \max(0,\langle\wt_j,\bvsigma_{\text{in}(j)}\rangle)$.
    \item \textbf{Decomposition of network output.}\\
    Let $\bvsigma_{j\leadsto\bullet}$ denote the sum of active path weights from $j$ to the output layer. 
    The output of the network decomposes as
    \begin{equation}
      \bvsigma_\text{out} 
      = 
      \begin{cases}
        \bvsigma_{j\leadsto\bullet}\cdot 
      \left\langle\wt_j,\bvsigma_{\text{in}(j)}\right\rangle 
      + \bsigma^{\activ\setminus\{j\}}_\text{out} 
      & \text{if player $j$ active} \\
      \bsigma^{\activ\setminus\{j\}}_\text{out}  & \text{else}
      \end{cases}      
    \end{equation}
    where $\bsigma^{\activ\setminus\{j\}}_\text{out}$ is the sum over active paths from sources to outputs \emph{that do not intersect $j$}.
    \item \textbf{Backpropagated errors.}\\
    Suppose the network is equipped with error function $\ell(\bvsigma_\text{out},\y)$.
    Let $\subg := \grad_\text{out}\ell(\bvsigma_\text{out},\y)$ denote the gradient of $\ell$. The backpropagated error signal received by unit $j$ is
    $\delta_j = \left\langle\subg,\bvsigma_{j\leadsto\bullet}\right\rangle$.
    \item \textbf{Error gradients.}\\
    Finally,
      \begin{align}
        \left\langle\grad_{\wt_j}\ell(\bvsigma_\text{out},\y),\wt_j\right\rangle
                & =\left\langle\subg,\bvsigma_{\text{out}}
                -\bsigma^{\activ\setminus\{j\}}_\text{out}
                \right\rangle\\
                & = \left\langle\subg,\bvsigma_{j\leadsto out}\right\rangle
                \cdot \left\langle\wt_j,\bvsigma_{in(j)}\right\rangle
                = \delta_j\cdot\varsigma_{\bullet\leadsto j}
      \end{align}
  \end{enumerate}
\end{prop}

\begin{proof}
  Direct computation, paralleling proposition~\ref{prop:lin}.
\end{proof}

The output of a rectifier network is a piecewise polynomial function of its weights. To see this, observe that the output of a rectifier net is the sum over all active paths from the input layer to the output unit, see also \cite{choromanska:14}. 

The next step is to construct a game played by the units of the neural network. It turns out there are two ways of doing so:
  
\begin{defn}[path-sum games]\label{def:bp_game}\eod
  The set of players is $\{0\}\cup [N]$. The zeroth player corresponds to the environment and is always active. The environment plays labeled datapoints $\wt_0=(\wt_\text{source},\y)\in\bR^{|\text{source}|+|\text{label}|}$ and suffers no loss. The remaining $N$ players correspond to non-source units of $G$. Player $j$ plays weight vector $\wt_j$ in compact convex $\ccH_j\subset\bR^{d_j}$.
  
  The losses in the two games are:
  \begin{itemize}
    \item \textbf{Path-sum prediction game ($\predg$).}\\
    Player $j$ incurs $\ell_j\big(\wt\big) := \ell(\bvsigma_\text{out},\y)$
    \emph{when active} and no loss when inactive.\\
    \item \textbf{Path-sum gradient game ($\gradg$).}\\
    Player $j$ incurs $\left\langle\grad\ell_j,\wt_j\right\rangle$ \emph{when active}, where $\grad\ell_j:=\grad_{\wt_j}\ell_j$, and no loss when inactive.
  \end{itemize}  
\end{defn}
$\predg$ and $\gradg$ are analogs of prediction with expert advice and the hedge setting. In the hedge setting, players receive linear losses and choose actions from the simplex; in $\gradg$, players receive linear losses. The results below hold for both games, although our primary interest is in $\predg$. Note that $\gradg$ has the important property that the loss of player $j$ is a \emph{linear} function of player $j$'s action when it is active:
\begin{equation}
  \left\langle\grad\ell_j,\wt_j\right\rangle = \begin{cases}
    \delta_j \cdot \langle\wt_j,\bvsigma_{in(j)}\rangle & \text{if $j$ active}\\
    0 & \text{else}.
  \end{cases}
\end{equation}
Finally, observe that the regret when playing $\gradg$ upper bounds $\predg$, since regret-bounds for linear losses are the worst-case amongst convex losses.

\begin{rem}[minibatch games]\eod
  It is possible to construct \emph{batch} or \emph{minibatch} games, by allowing the environment to play sequences of moves on each round. 
\end{rem}

\begin{prop}[path-sum games are gated games]\label{prop:net_circ}\eod
  $\predg$ and $\gradg$ are gated games if the error function $\ell(\bvsigma_\text{out},\y)$ is convex in its first argument. That is, rectifier nets are gated games.
\end{prop}

The gating structure is essential; path-sum games are \emph{not} convex, even for rectifiers with the mean-squared error: composing a rectifier with a quadratic can yield the nonconvex function $f(x)=(\max(0,x)-1)^2$. Even simpler, the negative of a rectifier is not convex.

\begin{proof}
  It is required to show that the losses under $\predg$ and $\gradg$, that is $\ell_j$ and $\langle\grad\ell_j,\wt_j\rangle$, are convex functions of $\wt_j$ when player $j$ is active. Clearly each loss is a scalar-valued function.

  By proposition~\ref{prop:struc}.2, when player $j$ is active the network loss has the form
  \begin{align}
    \ell_j(\wt) & = \ell\left(\bvsigma_{j\leadsto\bullet}\cdot 
      \left\langle\wt_j,\bvsigma_{\text{in}(j)}\right\rangle 
      + \bsigma^{\activ\setminus\{j\}}_\text{out},\y\right)\\
    & = \ell\left(c_1\cdot 
      \left\langle\wt_j,\bvsigma_{\text{in}(j)}\right\rangle 
      + c_2,\y\right).
  \end{align}
  The terms $\bvsigma_{j\leadsto\bullet}$, $\bvsigma_{\text{in}(j)}$ and $\bsigma^{\activ\setminus\{j\}}_\text{out}$ are all constants with respect to $\wt_{j}$. Thus, the network loss is an affine transformation of $\wt_j$ (dot-product followed by multiplication by a constant and adding a constant) composed with a convex function, and so convex. 

  By proposition~\ref{prop:struc}.4, the gradient loss has the form
  \begin{equation}
    \grad_{\wt_j}\ell_j(\wt) = \left\langle\subg,\bvsigma_{j\leadsto out}\right\rangle
                \cdot \bvsigma_{in(j)}
                = c_1\cdot \bvsigma_{in(j)}
  \end{equation}
  when player $j$ is active -- which is linear in $\wt_j$ since all the other terms are constants with respect to $\wt_j$.
\end{proof}

\begin{rem}[dependence of loss on other players]\eod
  We have shown that it is a convex function of player $j$'s action, when player $j$ is active. Note that: (i) the loss of player $j$ depends on the actions chosen by other players in the game and (ii) the loss is not a convex function of the joint-action of all the players. It is for these reasons that the game-theoretic analysis is essential. 
\end{rem}

The proposition does not merely hold for toy cases. The next section extends the result to maxout units, DropOut, DropConnect, and convolutional networks with shared weights and max-pooling. Proposition~\ref{prop:net_circ} thus applies to convolutional networks as they are used in practice. Finally, note that proposition~\ref{prop:net_circ} \emph{does not} hold for leaky rectifier units \cite{maas:13} or units that are not piecewise linear, such as sigmoid or $\tanh$.

\subsection{Convolutional Networks}
\label{sec:conv}

We extend proposition~\ref{prop:net_circ} from rectifier nets to convnets. 

\begin{prop}[convnets are gated games]\label{prop:conv}\eod
  Let ${\mathcal N}$ be a convolutional network with any combination of linear, rectifier, maxout and max-pooling units. Then, ${\mathcal N}$ is a gated game.
\end{prop}

The proof consists in identifying the relevant players and gates for each case (maxout units, max-pooling, weight-tying in convolutional layers, dropout and dropconnect) in turn. We sketch the result below.

\subsubsection{Maxout units}

Maxout units were introduced in \cite{goodfellow:13} to deal with the problem that rectifier units sometimes saturate at zero resulting in them being insufficiently active and to complement dropout.  A \emph{maxout} unit has $k_j$ weight vectors $\wt_{j,1},\ldots, \wt_{j,k_j} \in\bR^{d_j}$ and, given input $\bbf_j\in\bR^{d_j}$, outputs 
\begin{equation}
  \text{maxout:}\qquad m(\bbf_j):= \max_{c\in[k_j]}\; \big\langle\wt_{j,c},\bbf_j\big\rangle
\end{equation}

Construct a new graph, $\tilde{G}$, which has: one node per input, linear and rectifier unit; and $k_j$ nodes per maxout unit. Players correspond to nodes of $\tilde{G}$ and are denoted by greek letters. The extended graph inherits its edge structure from $G$: there is a connection between players $\alpha\rightarrow\beta$ in $\tilde{G}$ iff the underlying units in $G$ are connected. Path weights and path-sums are defined exactly as before, except that we work on $\tilde{G}$ instead of $G$. The definition of active units is modified as follows:

  The set $\activ$ of \textbf{active players for maxout units} is defined inductively. Let $\activ_k = \{$active players with longest source-path $\leq k\}$. Source players are active ($k=0$). 

  Player $\beta$ with source-path-length $k+1$ is active if
  \begin{itemize}
    \item it corresponds to a linear unit; or 
    \item a rectifier with $\sigma^{\activ_k\cup\{\beta\}}_\beta>0$; or
    \item a maxout unit with $\sigma^{\activ_k\cup\{\beta\}}_{\beta}>\sigma^{\activ_k\cup\{\alpha\}}_{\alpha}$ for all $\alpha$ corresponding to the same maxout unit. 
   \end{itemize}

\subsubsection{Max-pooling}
Max-pooling is heavily used in convnets to as a form of dimensionality reduction. A \emph{max-pooling} unit $j$ has no parameters and outputs the maximum of the outputs of the units from which it receives inputs: 
\begin{equation}
  \text{max-pooling:}\qquad\max_{\{i:i\rightarrow j\}} \sigma^\activ_i
\end{equation}
Gates can be extended to max-pooling by adding the condition that, to be active, the output of unit $i$ must be greater than any other unit $i'$ that feeds (directly) into the same pooling unit.

A unit may thus produce an output and still count as inactive because it is ignored by the max-pooling layer, and so has no effect on the output of the neural network. In particular, units that are ignored by max-pooling do not update their weights under backpropagation.

\subsubsection{Convolutional layers}

Units in a convolutional layer share weights. Obversely to maxout units, each of which corresponds to many players, weight-sharing units correspond to a single composite player. 

Suppose that rectifier units $j_1,\ldots, j_L$ share weight vector $\wt_j$. Let $\activ_{<j}$ denote active players in lower layers and define
\begin{equation}
  \fA_j:=\left\{\alpha\in [L]: \left\langle\wt_j,\bsigma^{\activ_{<j}}_{\text{in}(j_\alpha)}\right\rangle>0\right\}
\end{equation}
Component $\alpha$ in layer $j$ is active if $\alpha\in\fA_j$. Notice that, since players correspond to many units, two players may be connected by more than one edge. Player $j$ is active if any of its components is active, i.e. if $|\fA_j|>0$. The output of player $j$ is the sum of its active components:
\begin{equation}  
  \varsigma_j = \left\langle\wt_j,\sum_{\alpha\in\fA_j}\bsigma^{\activ_{<j}}_{\text{in}(j_\alpha)}\right\rangle.
\end{equation}
The loss incurred by player $j$ is per Definition~\ref{def:bp_game}.

\subsubsection{Dropout and Dropconnect}

In training with dropout \cite{srivastava:14}, units become inactive with some probability (typically $\frac{1}{2}$) during training. In other words, there is a stochastic component to whether or not a player is active. Gated games are easily extended to incorporate dropout by allowing gates to switch off stochastically. That is, the gating function takes values in the set of distributions over the set of units, $\activ:\ccH\rightarrow \prob(\text{units}(G))$, from which the active units are sampled.

Dropconnect is a refinement of dropout where connections, instead of units, are dropped out during training \cite{wan:13}. Dropconnect requires extending the notion of gate so that its range is distributions over subsets of edges instead of subsets of units: $\activ:\ccH\rightarrow \prob(\text{edges}(G))$. 

\section{Deep Online Convex Optimization}
\label{sec:doco}

We now explore some implications of the connection between path-sum games and deep learning. Theorem~\ref{thm:critical} shows that the convergence rates of gated forecasters in a path-sum game (that is, a rectifier net or convnet) controls the convergence rate of the network as a whole. As an immediate corollary, we obtain the first convergence rates for gradient descent on rectifier nets. A second corollary, that is of more conceptual than practical importance, shows that the signal underpinning the correlated equilibrium can be used to describe the representation learned by the neural network. Finally, we present an algorithm with logarithmic regret.

\subsection{A local-to-global convergence guarantee}

Our main result is that gated-regret controls the rate of convergence to critical points in rectifier convnets with a loss function that is convex in the output of the net. The result holds assuming weight vectors are restricted to compact convex sets. Weights are not usually hard-constrained when training neural networks although they are frequently regularized. Hinton has recently argued that the weights of rectifier layers quickly stabilise on similar values, suggesting this is not an issue in practice.\footnote{For example, see \url{http://bit.ly/1KN8e85} starting at 24:00}

It is important to note that the theorem applies to convolutional nets as used in practice. Rectifiers have replaced sigmoids as the nonlinearity of choice as they consistently yield better empirical performance.  Loss functions are convex in almost all applications: the logistic loss, hinge loss, and mean-squared error are all convex functions of output weights.

\begin{thm}[local-to-global convergence guarantee]\label{thm:critical}\eod
  Let ${\mathcal N}$ be a rectifier convnet trained by using backpropagation to compute gradients and a no-regret algorithm (such as gradient descent, Adagrad, mirror descent) to update weights given the gradients.
  
  Let $\gregret_j(T)$ denote the gated-regret of player $j$ after $T$ rounds. Then, the empirical distribution $\hat{\prob}^T$ of the weights arising during training network ${\mathcal N}$ over $T$ rounds converges to a correlated equilibrium. That is, 
    \begin{equation}
      \label{eq:main_conv}
      \expec_{\wt\sim \hat{\prob}_j^T}\big[\loss_j(\wt)\big] 
      \leq \min_{\wt_j\in\ccH_j}\expec_{\wt\sim \hat{\prob}_j^T}\big[\loss_j(\wt_j,\wt_{-j})] + \gregret_j(T)      
    \end{equation}
  for all players $j$. Consequently, the gated regret of the players controls the rate of convergence of backpropagation to critical points when training rectifier nets.
\end{thm}
An important class of games, introduced by Monderer and Shapley in \cite{monderer:96}, is \emph{potential games}. A game is a potential game if the loss functions of all the players arise from a single function, referred to as the potential function. Rectifier nets are gated potential games where the potential function is the loss of the network. That is, the loss incurred by each player, when active, is the loss of the network. Potential games are more amenable to analysis and computation than general games. Local minima of the potential function are pure Nash equilibria. Moreover, simple algorithms such as fictitious play and regret-matching converge to Nash equilibria in potential games \cite{hofbauer:02,hart:03}.

\begin{proof}
  Propositions~\ref{prop:corr_eqm}, \ref{prop:net_circ} and \ref{prop:conv} together imply Eq.~\eqref{eq:main_conv}. 

  The output $f_\Wt(\x)$ of a neural net is a continuous piecewise polynomial function of its weights, recall remark after prop~\ref{prop:struc}. The potential function $\ell(f_\Wt(\x),\y)$ of a neural net is the composite of a piecewise polynomial and convex function.
  It follows that no-regret algorithms will either converge to a point where the gradient is zero or to a point where the gradient does not exist. Thus, the network converges to a correlated equilibrium that is a Dirac distribution concentrated on a critical point of the loss function.
\end{proof}

The theorem provides the first rigorous justification for applying convex methods to convnets: although they are not convex, individual units perform convex optimizations when active. The theorem provides a generic conversion from regret-guarantees for convex methods to convergence rates for rectifier networks.
Corollaries~\ref{cor:ogd} and \ref{cor:log} provide algorithms for which $\gregret(T) \leq O(\frac{1}{\sqrt{T_\text{active}}})$ and $\gregret(T) \leq O(\frac{\log T_\text{active}}{T_\text{active}})$ respectively.

\subsection{Gradient descent}

A special case of theorem~\ref{thm:critical} is when the no-regret algorithm is gradient descent, see algorithm~\ref{alg:ebp}. The algorithm differs from standard backpropagation by introducing a projection step
\begin{equation}
  \proj_{\ccH_j}(\wt') := \argmin_{\wt\in\ccH_j}\|\wt-\wt'\|_2^2
\end{equation}
that forces the updated weight to lie in the set of actions available to player $j$. If the diameter of $\ccH_j$ is sufficiently large then the projection step makes no practice. It can be thought of as analogous to gradient clipping, which is sometimes used when training neural nets.

\begin{cor}[convergence for gradient descent]\label{cor:ogd}\eod
  Suppose a neural network has a loss function that is convex in its output. Suppose that $\ccH_j\subset \bR^{d_j}$ has diameter $D$. Further suppose that that the backpropagated errors received by $j$ and the inputs to $j$ are bounded by $\max_t|\delta_j^t|\leq B$ and $\max_t \|\bvsigma_{in(j)}^t\|\leq G$ respectively.

  Then unit $j$'s gated-regret under online gradient descent is bounded by
  \begin{equation}
    \gregret_{\texttt{Backprop}}(T)
    \leq \frac{3}{2}DGB\frac{1}{\sqrt{T_j}}
  \end{equation}
  where $T_j\leq T$ is the number of rounds where $j$ is active.
\end{cor}

The learning rate in algorithm~\ref{alg:ebp} decays according to the number of \emph{active} steps $t_j$ rather than the number of steps $t$. An important insight of the gated-game formulation is that learning only occurs on active rounds.

\begin{proof}
  By Lemma~\ref{prop:struc}, weight updates under error backpropagation coincide with players performing online gradient descent, when active, on either loss in Definition~\ref{def:bp_game}.  The gradient $\grad_{\wt_j}\ell_j$ depends on player $j$'s input, upper bounded by $G$, and its backpropagated error, upper bounded by $B$. The result follows from a standard analysis of online gradient descent \cite{zinkevich:03}.
\end{proof}

\addtocounter{algocf}{-1}
\renewcommand*{\algorithmcfname}{Algorithm}
\begin{algorithm}
  \caption{\textsc{$\texttt{Backprop}$ (Error Backpropagation)}}\label{alg:ebp}
  \DontPrintSemicolon
  \SetKwInOut{Input}{input}
  \Input{
    Learning rates $\{\eta^t = \frac{D}{BG\sqrt{t_j}}\}$\;
  } 
  Pick $\wt^1$ in $\ccH_j$\;
  \For{rounds $t =1, 2, \ldots$}{
    Input $\x^t$ revealed\;
    \If{\texttt{gate}$(\wt^t,\x^t)$}{
      Backpropagated error $\delta^t$ revealed\;
      $\wt^{t+1}  \longleftarrow 
      \proj_{\ccH_j}\left(\wt^t - \eta^t\cdot \delta^t\cdot \x^t\right)$\;
    }
  }
\end{algorithm}

  The bound is a function of constants, $B$ and $G$, that depend on the behavior of other units in the neural net. The dependence arises for any gradient descent based algorithm where the weight updates depend on the backpropagated error.  The corollary precisely characterises the dependence.

\subsection{Signals and representations}

Recall that a correlated equilibrium requires a \emph{signal} to guide the behavior of the players. In the case of a rectifier net, the relevant signal is the emprical distribution over the joint actions of the players. As a corollary of theorem~\ref{thm:critical}, we show that the signal provides a compact description of the representations learned by deep networks. There is thus a direct connection between correlated equilibria and representation learning.

Given a distribution $\prob$ on the set $\ccH$ of joint actions (recall that a joint action in $\predg$ specifies the input to the network, its label, and every weight vector), define the \emph{expected gain} of player $j$ as
\begin{equation}
  \label{eq:gain}
  \text{gain:} \qquad\gain_j(\wt_j;\prob) := \expec_{\wt_{-j}\sim\prob} \Big[-\loss_j\big(\wt_j,\wt_{-j}\big)\Big],  
\end{equation}
where $\wt_{-j}$ is moves by players other than $j$. Note that in Eq.~\eqref{eq:gain}, the moves of all players except $j$ are drawn from $\prob$, which determines which players are active; player $j$'s move (if active) is treated as a free variable.

Let $\hat{\prob}^{t}$ denote the empirical distribution -- or \emph{signal} in game-theoretic terminology -- on joint actions up to round $t$ of a neural network trained by error backpropagation, and $\hat{\prob}^{t}_j$ the empirical signal observered by player $j$. For notational convenience, it is useful to incorporate the learning rate, number of rounds and initial weight vector into the gain, and define 
\begin{equation}
  \label{eq:empirical_gain}
  \hat{\gain}_j(\wt_j) := \eta T_j\cdot \gain_j(\wt;\hat{\prob}_j) + \langle\wt_j,\wt^1_j\rangle,
\end{equation}
where $T_j$ is the number of rounds where unit $j$ is active. We then have

\begin{cor}[signals $\leftrightarrow$ representations]\label{cor:rep_learning}\eod
  Construct the empirical gain $\hat{G}_j$ of unit $j$ after $T$ rounds from the signal (empirical distribution) per \eqref{eq:empirical_gain}.
  Then if a rectifier net implements gradient descent with fixed learning rate $\eta$ and unconstrained weights, it holds that
  \begin{itemize}
    \item unit $j$'s weight vector at time $T+1$ is the gradient of the gain:
    \begin{equation}
      \wt^{T+1}_j 
      = \grad_j \hat{\gain}_j     
    \end{equation}
    \item the output of unit $j$ on round $T+1$ is the directional derivative of the gain w.r.t. $j$'s input:
    \begin{equation}
      \varsigma_j^{T+1} 
      = \begin{cases}
        D_{\bvsigma^{T+1}_{\text{in}(j)}} \hat{\gain}_j  & \text{if positive}\\
        0 & \text{else.}
      \end{cases}
    \end{equation}
  \end{itemize}
\end{cor}
Corollary~\ref{cor:rep_learning} succinctly describes the representations learned by a neural network via the game-theoretic notation developed above. The corollary does not eliminate the complexity of deep representations. Rather, it demonstrates their direct dependence on the empirical signal $\hat{\prob}$, which is itself an extremely complicated object.

\subsection{Logarithmic convergence}
\label{sec:sl}

As a second application of theorem~\ref{thm:critical}, we adapt the Online Newton Step (ONS) algorithm \cite{hazan:07a} to neural networks, see $\nprop$ in Algorithm~\ref{alg:np}. Newton's method is computationally expensive since it involves inverting the Hessian. In particular, Online Newton Step scales as $O(n^4)$ where $n$ is the dimension \cite{koren:13}. Moreover, quasi-newton method tends to converge on saddle points. A naive implementation of a quasi-newton method in neural networks based on the global hessian is therefore problematic: the number of parameters is huge; and saddle points are abundant \cite{baldi:12,dauphin:14}.

The $\nprop$ algorithm sidesteps both problems since it is implemented \emph{unit-wise}. The computational cost is reduced, since an approximation to a local Hessian is computed for each unit. Thus, the computational cost scales as quadratically with the largest layer, rather than with the network. Similarly, since $\nprop$ is implemented unit-wise, the Newton-approximation is not exposed to curvature of the neural net. Instead, $\nprop$ simultaneously leverages the linear structure of active path-sums and the exp-concave structure (curvature) of the external loss.

Let $\ccH\subset \bR^d$ be a non-empty compact convex set. A function $f:\ccH\rightarrow \bR$ is \textbf{$\alpha$-exp-concave} if $e^{-\alpha f(\wt)}$ is a concave function of $\wt\in \ccH$. Many commonly used loss functions are exp-concave, including the mean-squared error, $(\langle\wt,\x\rangle-y)^2$, the log loss, $-\log\langle\wt,\x\rangle$, and the logistic loss, $\log(1+e^{-y\langle\wt,\x\rangle})$, for suitably restricted $\wt,\x$ and $y$.

Recall that 
\begin{equation}
  \proj_{\ccH,\bA}(\wt) 
  := \argmin_{\vt\in\ccH} \big\langle\vt-\wt,\bA\cdot(\vt-\wt)\big\rangle.
\end{equation}

Given vectors $\bu$ and $\bv$, let $\bu\otimes \bv$ denote their outerproduct. If $\bu$ and $\bv$ are $m$ and $n$ dimensional respectively then $\bu\otimes \bv$ is a $(m\times n)$-matrix.

\begin{cor}[$\nprop$ has logarithmic gated-regret]\label{cor:log}\eod
  Suppose that a neural network has loss function $\ell$ that is $\alpha$-exp-concave in its output. Suppose that $\ccH_j\subset\bR^{d_j}$ has diameter $D$. Further suppose that the backpropagated errors and inputs to $j$ are bounded by $\max_t|\delta^t|\leq B$ and $\max_{t}\|\x^t\|\leq G$ respectively.   

  Then, unit $j$'s gated-regret under $\nprop$ is bounded by
  \begin{equation}
    \gregret_{\nprop}(T) \leq 5d_j\left(\frac{1}{\alpha}+BDG\right)\frac{\log T_j}{T_j}
  \end{equation}
  where $T_j\leq T$ is the number of rounds that $j$ is active and $d_j$  is its indegree.
\end{cor}

We first prove the following lemma.

\begin{lem}\label{lem:hazan}
  Let $f:X\rightarrow \bR$ be an $\alpha$-exp-concave function. Suppose that $\ccH\subset\bR^n$ is a nonempty compact convex set with $\diam(\ccH)=D$, and that $\bA$ and $\bbb$ are a $(d\times n)$-matrix and a $d$-vector satisfying $\bA\cdot \ccH +\bbb\subset X$.  Suppose that $\max_{\x\in X}\|\bA^\intercal\grad f(\x)\|\leq E$. 

  Define $g:\ccH\rightarrow \bR$ as $g(\wt):=f(\bA \wt +\bbb)$. If $\beta\leq \frac{1}{2}\min\{\frac{1}{4DE},\alpha\}$ then for all $\wt,\bv\in\ccH$ it holds that
  \begin{align}
    \label{eq:key_bound}
    g(\bv)  \geq & \; g(\wt) + \langle\grad g(\wt),\bv-\wt\rangle \\
     & + \frac{\beta}{2}\left\langle \wt-\bv, \grad g(\wt) \otimes\grad g(\wt) \cdot(\wt-\bv)\right\rangle.
  \end{align}
\end{lem}

\begin{proof}
  It is shown in Lemma~3 of \cite{hazan:07a} that, since $f$ is $\alpha$-exp-concave and $2\beta\leq \alpha$,
  \begin{equation}
    \label{eq:key}
    f(\x) - \frac{1}{2\beta}\log\big(1-2\beta \langle\grad f(\x),\y-\x\rangle\big) \leq f(\y).
  \end{equation}  
  By the chain rule, $\grad g(\wt) = \bA^\intercal \grad f(\x)$, and so Eq.~\eqref{eq:key} can be rewritten as
  \begin{equation}
    g(\wt) - \frac{1}{2\beta}\log\big(1-2\beta \langle\grad g(\wt),\bv-\wt\rangle\big) \leq g(\bv).
  \end{equation}
  By construction, $|2\beta \langle\grad g(\wt),\bv-\wt\rangle|\leq \frac{1}{4}$ and the result follows by the reasoning in \cite{hazan:07a}.
\end{proof}

We are now ready to prove the Theorem.

\begin{proof}
  The proof follows the same logic as Theorem~2 in \cite{hazan:07a} after replacing  Lemma~3 there with our Lemma~\ref{lem:hazan}. We omit details, except to show how the setting of Lemma~\ref{lem:hazan} connects to neural networks. Let
  \begin{equation}
    \x^t := \bvsigma_{\text{in}(j)}^t,
    \qquad
    \bpi^t := \bvsigma_{j\leadsto \text{out}}^t
    \qquad\text{and}\qquad
    \bbb^t := \bsigma^{\activ\setminus\{j\}}_\text{out}.
  \end{equation}
  Let the $(d\times m)$-dimensional matrix $\bA^t:= \bpi^t\otimes (\x^t)^\intercal$ denote the outer product. By Lemma~\ref{prop:struc}.2,
  \begin{equation}
    \bvsigma_\text{out}^t = \bA^t\cdot \wt_j^t + \bbb^t.
  \end{equation}
  Since 
  \begin{equation}
    \bA^\intercal \grad\loss = \x^t\cdot \langle\bpi^t,\grad\loss\rangle = \delta_j^t\cdot \x^t,
  \end{equation}
  we have that
  \begin{equation}
    \|\bA^\intercal\grad\loss\|\leq |\delta_j^t|\cdot \|\x^t\|\leq BG,
  \end{equation}
  and the remainder of the argument is standard.
\end{proof}

To the best of our knowledge, $\nprop$ is the first logarithmic-regret algorithm for neural networks. $\nprop$ is computationally more efficient than $2^\text{nd}$ order methods, since it does not require computing the Hessian, and there are efficient ways to iteratively compute $(\bA^t)^{-1}$ without directly inverting the matrix. Nevertheless, $\nprop$'s memory usage and computational complexity are prohibitive \cite{koren:13}; it is worth investigating whether there are more efficient algorithms that achieve logarithmic regret in this setting, for example based on the Sketched Online Newton algorithm proposed by Luo \emph{et al} \cite{luo:16} which has linear runtime. Finally, $\nprop$ does not take advantage of the fact that some experts are inactive on each round, suggesting a second direction in which it may be improved.

\renewcommand*{\algorithmcfname}{Algorithm}
\begin{algorithm}
  \caption{\textsc{$\nprop$ (Newton Backpropagation)}}\label{alg:np}
  \DontPrintSemicolon
  \SetKwInOut{Input}{input}
  \Input{
    $\beta\longleftarrow \frac{1}{2}\min\{\frac{1}{4BGD},\alpha\}$\;    
  } 
  $\bA^0\longleftarrow \frac{1}{\beta^2 D^2} \bI_{d_j}$\;
  Pick $\wt^1$ in $\ccH_j$\;
  \For{rounds $t =1, 2, \ldots$}{
    Input $\x^t$ revealed\;
    \If{\texttt{gate}$(\wt^t,\x^t)$}{
      Backpropagated error $\delta^t$ revealed\;
      $\bA^{t} \longleftarrow \bA^{t-1} + (\delta^t)^2\cdot \x^t\otimes\x^t$\;    
      $\wt^{t+1}  \longleftarrow \proj_{\ccH_j,\bA^{t}}\left(\wt^t - \frac{\delta^t}{\beta} (\bA^t)^{-1}\x^t\right)$\;
    }
  }
\end{algorithm}

\subsection{Conditional computation}
\label{sec:ls}

Convnets are path-sum games played between gated convex players. The criterion for activating a unit is either a $\max$ operator (rectifiers, maxout units, and max-pooling) or random (dropout and dropconnect). These have been shown to work well in practice. It is nevertheless natural to question whether they are \emph{optimal}. This section introduces a framework to tackle the question. Indeed 

Analyzing and optimizing the gates requires a new kind of player, \emph{Conditional Gate (CoG)}, that controls when players are active. Conditional Gate experiences regret about not activing the optimal subset of players. More precisely, a CoG activates a subset of players $F^t\subset F$ on each round. The CoG's context is the weights of the players and their inputs. In $\predg$, a CoG incurs scalar loss $\loss(\bvsigma_\text{out},\y)$. In $\gradg$, a CoG incurs loss vector $\left(\grad \ell_i^t\right)_{i\in F^t}$. 

\addtocounter{algocf}{-1}
\renewcommand*{\algorithmcfname}{Setting}
\begin{algorithm}
  \caption{\textsc{Conditional Gate (CoG)}}\label{alg:sm}
  \DontPrintSemicolon
  \SetKwInOut{Input}{input}
  \Input{
    Set $F$ of players\;
    Function class $\cF\subset\{\varphi:\cC\rightarrow \cP(F)\}$
  } 
  \For{
  rounds $t =1, 2, \ldots$}{
    Weights $(\wt^t_i)_{i\in F}$ and inputs $(\x^t_i)_{i\in F}$ to players are revealed\;
    CoG activates subset $F^t\subset F$ of players\;
    CoG incurs loss $\ell^t(F^t)$\;
  }
\end{algorithm}

The setting is a bandit or partial monitoring since the CoG does not observe what the loss would have been had other players been active. A CoG has access to a class of functions $\cF\subset\{\varphi:\cC\rightarrow \cP(F)\}$ from contexts to sets of players. The CoG's \emph{regret} and \emph{pseudo-regret} are
\begin{align}
  \regret & = \frac{1}{T}\expec\max_{\varphi\in\cF} \left[\sum_{t=1}^T \ell^t(F^t) -\ell^t\big(\varphi(c^t)\big)\right]
  \quad\text{and}\\
  \pregret & = \frac{1}{T}\max_{\varphi\in\cF} \expec\left[\sum_{t=1}^T \ell^t(F^t) -\ell^t\big(\varphi(c^t)\big)\right],
\end{align}
where the expectation is over the CoG's actions (which are stochastic in general).

The next two examples sketch how to generalize rectifier and maxout units using conditional gates.
\begin{eg}[rectifiers]
  Instead of using the $\max$ function to activate a linear unit, equip it with a CoG that decides, on each round, whether or not to activate the unit based on context $(\wt, \x)$. If the unit is not woken, then CoG does not observe the loss. The setting generalizes apple-tasting \cite{helmbold:00} to contextual partial monitoring.
\end{eg}

\begin{eg}[maxout]
  Section~\ref{sec:conv} showed how to model a maxout unit as $k$ players. An adaptive-maxout unit is then $k$ players with a CoG that activates one of them on each round. The CoG is a contextual bandit: players are levers; their weights and inputs are context; the loss $\loss(\bvsigma_\text{out},\y)$ is a scalar that depends on the choice of active player. 
\end{eg}

There are currently no theoretically grounded off-the shelf methods for contextual partial monitoring. However, an efficient contextual bandit algorithm for the stochastic setting is provided in \cite{agarwal:14}. Recent work \cite{bacon:15,bengio:15e} has implemented conditional computation in neural nets using REINFORCE \cite{williams:92} to train the policy. Although the approach has been shown to work in practice, there are currently no performance guarantees available. Developing a solid understanding of conditional computation in neural nets is an important open problem.

\section{Conclusion}

The paper has initiated a game-theoretic analysis of convolutional networks. The key observation is that the nonlinearities found in modern convnets (rectifiers, maxout, max-pooling) are gates that control whether or not the linear operations performed by units contribute to the network's output. Gated games formalize the role of gating. Path-sum games succinctly express the dynamics of convolutional networks. Reformulating error backpropagation as a path-sum game yields the first convergence rates for modern convnets. It also provides a solid foundation on top of which ideas from online convex optimization, game theory and deep learning can be combined.

	\section*{Acknowledgements}
	I am grateful to Jacob Abernethy, Samory Kpotufe and Brian McWilliams for useful conversations.



\begin{thebibliography}{10}
\providecommand{\url}[1]{#1}
\csname url@samestyle\endcsname
\providecommand{\newblock}{\relax}
\providecommand{\bibinfo}[2]{#2}
\providecommand{\BIBentrySTDinterwordspacing}{\spaceskip=0pt\relax}
\providecommand{\BIBentryALTinterwordstretchfactor}{4}
\providecommand{\BIBentryALTinterwordspacing}{\spaceskip=\fontdimen2\font plus
\BIBentryALTinterwordstretchfactor\fontdimen3\font minus
  \fontdimen4\font\relax}
\providecommand{\BIBforeignlanguage}[2]{{%
\expandafter\ifx\csname l@#1\endcsname\relax
\typeout{** WARNING: IEEEtran.bst: No hyphenation pattern has been}%
\typeout{** loaded for the language `#1'. Using the pattern for}%
\typeout{** the default language instead.}%
\else
\language=\csname l@#1\endcsname
\fi
#2}}
\providecommand{\BIBdecl}{\relax}
\BIBdecl

\bibitem{lecun:15}
Y.~LeCun, Y.~Bengio, and G.~Hinton, ``Deep learning,'' \emph{Nature}, vol. 521,
  pp. 436--444, 2015.

\bibitem{werbos:74}
P.~J. Werbos, ``Beyond {R}egression: {N}ew {T}ools for {P}rediction and
  {A}nalysis in the {B}ehavioral {S}ciences,'' Ph.D. dissertation, Harvard,
  Cambridge MA, 1974.

\bibitem{rumelhart:86}
D.~E. Rumelhart, G.~E. Hinton, and R.~J. Williams, ``Learning representations
  by back-propagating errors,'' \emph{Nature}, vol. 323, pp. 533--536, 1986.

\bibitem{schmidhuber:15}
J.~Schmidhuber, ``Deep {L}earning in {N}eural {N}etworks: {A}n {O}verview,''
  \emph{Neural Networks}, vol.~61, pp. 85--117, 2015.

\bibitem{srivastava:14}
N.~Srivastava, G.~Hinton, A.~Krizhevsky, I.~Sutskever, and R.~Salakhutdinov,
  ``Dropout: {A} {S}imple {W}ay to {P}revent {N}eural {N}etworks from
  {O}verfitting,'' \emph{JMLR}, vol.~15, pp. 1929--1958, 2014.

\bibitem{deng:09}
J.~Deng, W.~Dong, R.~Socher, L.-J. Li, K.~Li, and L.~Fei-Fei, ``{Imagenet: A
  large-scale hierarchical image database},'' in \emph{CVPR}, 2009.

\bibitem{krizhevsky:12}
A.~Krizhevsky, I.~Sutskever, and G.~E. Hinton, ``Imagenet classification with
  deep convolutional neural networks,'' in \emph{Advances in Neural Information
  Processing Systems (NIPS)}, 2012.

\bibitem{zeiler:14}
M.~Zeiler and R.~Fergus, ``{Visualizing and Understanding Convolutional
  Networks},'' in \emph{ECCV}, 2014.

\bibitem{szegedy:15}
C.~Szegedy, W.~Liu, Y.~Jia, P.~Sermanet, S.~Reed, D.~Anguelov, D.~Erhan,
  V.~Vanhoucke, and A.~Rabinovich, ``{Going Deeper With Convolutions},'' in
  \emph{CVPR}, 2015.

\bibitem{he:15}
K.~He, X.~Zhang, S.~Ren, and J.~Sun, ``{Delving Deep into Rectifiers:
  Surpassing Human-Level Performance on ImageNet Classification},'' Microsoft
  Research, http://arxiv.org/abs/1502.01852, Tech. Rep., 2015.

\bibitem{jarrett:09}
K.~Jarrett, K.~Kavukcuoglu, M.~Ranzato, and Y.~LeCun, ``What is the {B}est
  {M}ulti-{S}tage {A}rchitecture for {O}bject {R}ecognition?'' in \emph{Proc.
  International Conference on Computer Vision (ICCV)}, 2009.

\bibitem{nair:10}
V.~Nair and G.~Hinton, ``{R}ectified {L}inear {U}nits {I}mprove {R}estricted
  {B}oltzmann {M}achines,'' in \emph{ICML}, 2010.

\bibitem{glorot:11}
X.~Glorot, A.~Bordes, and Y.~Bengio, ``{D}eep {S}parse {R}ectifier {N}eural
  {N}etworks,'' in \emph{Proc. 14th {I}nt {C}onference on {A}rtificial
  {I}ntelligence and Statistics (AISTATS)}, 2011.

\bibitem{maas:13}
A.~L. Maas, A.~Y. Hannun, and A.~Ng, ``Rectifier {N}onlinearities {I}mprove
  {N}eural {N}etwork {A}coustic {M}odels,'' in \emph{Proceedings of the 30th
  {I}nternational {C}onference on {M}achine {L}earning (ICML)}, 2013.

\bibitem{dahl:13}
G.~E. Dahl, T.~N. Sainath, and G.~Hinton, ``Improving deep neural networks for
  {L}{V}{C}{S}{R} using rectified linear units and dropout,'' in
  \emph{I{E}{E}{E} {I}nt {C}onf on {A}coustics, {S}peech and {S}ignal
  {P}rocessing (ICASSP)}, 2013.

\bibitem{zeiler:13}
M.~D. Zeiler, M.~Ranzato, R.~Monga, M.~Mao, K.~Yang, Q.~V. Le, P.~Nguyen,
  A.~Senior, V.~Vanhoucke, J.~Dean, and G.~Hinton, ``On {R}ectified {L}inear
  {U}nits for {S}peech {P}rocessing,'' in \emph{ICASSP}, 2013.

\bibitem{pascanu:14}
R.~Pascanu, C.~Gulcehre, K.~Cho, and Y.~Bengio, ``On the number of inference
  regions of deep feed forward networks with piece-wise linear activations,''
  in \emph{ICLR}, 2014.

\bibitem{stoltz:07}
G.~Stoltz and G.~Lugosi, ``Learning correlated equilibria in games with compact
  sets of strategies,'' \emph{Games and Economic Behavior}, vol.~59, pp.
  187--208, 2007.

\bibitem{cybenko:89}
G.~Cybenko, ``Approximation by superposition of sigmoidal function,''
  \emph{Mathematics of Control, Signals, and Systems}, vol.~2, pp. 303--314,
  1989.

\bibitem{leshno:93}
M.~Leshno, V.~Y. Lin, A.~Pinkus, and S.~Schocken, ``Multilayer {F}eedforward
  {N}etworks {W}ith a {N}onpolynomial {A}ctivation {F}unction {C}an
  {A}pproximate {A}ny {F}unction,'' \emph{Neural Networks}, vol.~6, pp.
  861--867, 1993.

\bibitem{eldan:15}
R.~Eldan and O.~Shamir, ``{The Power of Depth for Feedforward Neural
  Networks},'' in \emph{arXiv:1512.03965}, 2015.

\bibitem{telgarsky:16}
M.~Telgarsky, ``{Benefits of depth in neural networks},'' in
  \emph{arXiv:1602.04485}, 2016.

\bibitem{daniely:16}
A.~Daniely, R.~Frostig, and Y.~Singer, ``{Toward Deeper Understanding of Neural
  Networks: The Power of Initialization and a Dual View on Expressivity},'' in
  \emph{arXiv:1602.05897}, 2016.

\bibitem{vapnik:95}
V.~Vapnik, \emph{The {N}ature of {S}tatistical {L}earning {T}heory}.\hskip 1em
  plus 0.5em minus 0.4em\relax Springer, New York NY, 1995.

\bibitem{anthony:99}
M.~Anthony and P.~L. Bartlett, \emph{Neural {N}etwork {L}earning: {T}heoretical
  {F}oundations}.\hskip 1em plus 0.5em minus 0.4em\relax Cambridge Univ Press,
  1999.

\bibitem{hardt:15}
M.~Hardt, B.~Recht, and Y.~Singer, ``{Train faster, generalize better:
  Stability of stochastic gradient descent},'' in \emph{arXiv:1509.01240},
  2015.

\bibitem{feng:16}
J.~Feng, T.~Zahavy, B.~Kang, H.~Xu, and S.~Mannor, ``{Ensemble Robustness of
  Deep Learning Algorithms},'' in \emph{arXiv:1602.02389}, 2016.

\bibitem{baldi:14}
P.~Baldi and P.~Sadowski, ``{The dropout learning algorithm},''
  \emph{Artificial Intelligence}, vol. 210, pp. 78--122, 2014.

\bibitem{wager:13}
S.~Wager, S.~Wang, and P.~Liang, ``{Dropout Training as Adaptive
  Regularization},'' in \emph{NIPS}, 2013.

\bibitem{vanerven:14}
T.~van Erven, W.~Kotlowski, and M.~Warmuth, ``{Follow the Leader with Dropout
  Perturbations},'' in \emph{COLT}, 2014.

\bibitem{ge:15}
R.~Ge, F.~Huang, C.~Jin, and Y.~Yuan, ``{Escaping From Saddle Points -- Online
  Stochastic Gradient Descent for Tensor Decomposition},'' in \emph{COLT},
  2015.

\bibitem{lee:16}
J.~D. Lee, M.~Simchowitz, M.~I. Jordan, and B.~Recht, ``{Gradient Descent
  Converges to Minimizers},'' in \emph{arXiv:1602.04915}, 2016.

\bibitem{dauphin:14}
Y.~Dauphin, R.~Pascanu, C.~Gulcehre, K.~Cho, S.~Ganguli, and Y.~Bengio,
  ``Identifying and attacking the saddle point problem in high-dimensional
  non-convex optimization,'' in \emph{NIPS}, 2014.

\bibitem{choromanska:14}
A.~Choromanska, M.~Henaff, M.~Mathieu, G.~B. Arous, and Y.~LeCun, ``The loss
  surface of multilayer networks,'' in \emph{Journal of Machine Learning
  Research: Workshop and Conference Proceeedings}, vol. 38 (AISTATS), 2015.

\bibitem{choromanska:15}
A.~Choromanska, Y.~LeCun, and G.~B. Arous, ``{Open Problem: The landscape of
  the loss surfaces of multilayer networks},'' in \emph{Journal of Machine
  Learning Research: Workshop and Conference Proceeedings}, vol. 40 (COLT),
  2015.

\bibitem{bolte:14}
J.~Bolte, S.~Sabach, and M.~Teboulle, ``Proximal alternating linearized
  minimization for nonconvex and nonsmooth problems,'' \emph{Math. Program.,
  Ser. A}, vol. 146, pp. 459--494, 2014.

\bibitem{li:15}
H.~Li and Z.~Lin, ``{Accelerated Proximal Gradient Methods for Nonconvex
  Programming},'' in \emph{NIPS}, 2015.

\bibitem{allen-zhu:16}
Z.~Allen-Zhu and E.~Hazan, ``{Variance Reduction for Faster Non-Convex
  Optimization},'' in \emph{arXiv:1603.05643}, 2016.

\bibitem{reddi:16}
S.~J. Reddi, S.~Sra, B.~P{\'o}cz{\'o}s, and A.~Smola, ``{Fast Incremental
  Method for Nonconvex Optimization},'' in \emph{arXiv:1603.06159}, 2016.

\bibitem{reddi:16a}
S.~J. Reddi, A.~Hefny, S.~Sra, B.~P{\'o}cz{\'o}s, and A.~Smola, ``{Stochastic
  Variance Reduction for Nonconvex Optimization},'' in \emph{arXiv:1603.06160},
  2016.

\bibitem{cesa:06}
N.~Cesa-Bianchi and G.~Lugosi, \emph{Prediction, {L}earning and {G}ames}.\hskip
  1em plus 0.5em minus 0.4em\relax Cambridge University Press, Cambridge UK,
  2006.

\bibitem{blum:97}
A.~Blum, ``{Empirical support for Winnow and Weighted-Majority algorithms:
  Results on a calendar scheduling domain},'' \emph{Machine Learning}, vol.~26,
  no.~1, pp. 5--23, 1997.

\bibitem{freund:97a}
Y.~Freund, R.~Schapire, Y.~Singer, and M.~Warmuth, ``Using and combining
  predictors that specialize,'' in \emph{STOC}, 1997.

\bibitem{balduzzi:gg}
D.~Balduzzi, ``{Grammars for Games: A Gradient-Based, Game-Theoretic Framework
  for Optimization in Deep Learning},'' \emph{Frontiers in Robotics and AI},
  vol.~2, no.~39, 2016.

\bibitem{neyshabur:15}
B.~Neyshabur, R.~Tomioka, and N.~Srebro, ``{Norm-Based Capacity Control in
  Neural Networks},'' in \emph{preprint}, 2015.

\bibitem{hazan:07a}
E.~Hazan, A.~Agarwal, and S.~Kale, ``Logarithmic regret algorithms for online
  convex optimization,'' \emph{Machine Learning}, vol.~69, pp. 169--192, 2007.

\bibitem{bacon:15}
P.-L. Bacon, E.~Bengio, J.~Pineau, and D.~Precup, ``Conditional computation in
  neural networks using a decision-theoretic approach,'' in \emph{2nd
  Multidisciplinary Conference on Reinforcement Learning and Decision Making
  (RLDM)}, 2015.

\bibitem{bengio:15e}
E.~Bengio, P.-L. Bacon, J.~Pineau, and D.~Precup, ``{Conditional Computation in
  Neural Networks for faster models},'' in \emph{arXiv:1511.06297}, 2015.

\bibitem{mnih:14}
V.~Mnih, N.~Heess, A.~Graves, and K.~Kavukcuoglu, ``Recurrent models of visual
  attention,'' in \emph{NIPS}, 2014.

\bibitem{gregor:15}
K.~Gregor, I.~Danihelka, A.~Graves, D.~J. Rezende, and D.~Wierstra, ``{DRAW: A
  Recurrent Neural Network For Image Generation},'' in \emph{ICML}, 2015.

\bibitem{hochreiter:97}
S.~Hochreiter and J.~Schmidhuber, ``Long {S}hort-{T}erm {M}emory,''
  \emph{Neural Comp}, vol.~9, pp. 1735--1780, 1997.

\bibitem{cho:14}
K.~Cho, B.~van Merri{\"e}nboer, C.~Gulcehre, D.~Bahdanau, F.~Bougares,
  H.~Schwenk, and Y.~Bengio, ``{Learning Phrase Representations using RNN
  Encoder--Decoder for Statistical Machine Translation},'' in \emph{EMNLP},
  2014.

\bibitem{balduzzi:16}
D.~Balduzzi and M.~Ghifary, ``{Strongly-Typed Recurrent Neural Networks},'' in
  \emph{arXiv:1602.02218}, 2016.

\bibitem{goodfellow:14}
I.~J. Goodfellow, J.~Pouget-Abadie, M.~Mirza, B.~Xu, D.~Warde-Farley, S.~Ozair,
  A.~Courville, and Y.~Bengio, ``{Generative Adversarial Nets},'' in
  \emph{NIPS}, 2014.

\bibitem{denton:15}
E.~Denton, S.~Chintala, A.~Szlam, and R.~Fergus, ``{Deep Generative Image
  Models using a Laplacian Pyramid of Adversarial Networks},'' in \emph{NIPS},
  2015.

\bibitem{Mnih:2015wq}
V.~Mnih, K.~Kavukcuoglu, D.~Silver, A.~A. Rusu, J.~Veness, M.~G. Bellemare,
  A.~Graves, M.~Riedmiller, A.~K. Fidjeland, G.~Ostrovski, S.~Petersen,
  C.~Beattie, A.~Sadik, I.~Antonoglou, H.~King, D.~Kumaran, D.~Wierstra,
  S.~Legg, and D.~Hassabis, ``Human-level control through deep reinforcement
  learning,'' \emph{Nature}, vol. 518, no. 7540, pp. 529--533, 02 2015.

\bibitem{silver:16}
D.~Silver, A.~Huang, C.~J. Maddison, A.~Guez, L.~Sifre, G.~van~den Driessche,
  J.~Schrittwieser, I.~Antonoglou, V.~Panneershelvam, M.~Lanctot, S.~Dieleman,
  D.~Grewe, J.~Nham, N.~Kalchbrenner, I.~Sutskever, T.~Lillicrap, M.~Leach,
  K.~Kavukcuoglu, T.~Graepel, and D.~Hassabis, ``Mastering the game of go with
  deep neural networks and tree search,'' \emph{Nature}, vol. 529, no. 7587,
  pp. 484--489, 01 2016.

\bibitem{ioffe:15}
S.~Ioffe and C.~Szegedy, ``{Batch normalization: Accelerating deep network
  training by reducing internal covariate shift},'' in \emph{arXiv:1502.03167},
  2015.

\bibitem{daskalakis:14}
C.~Daskalakis, A.~Decklebaum, and A.~Kim, ``Near-optimal no-regret algorithms
  for zero-sum games,'' \emph{Games and Economic Behavior}, vol.~92, pp.
  327--348, 2014.

\bibitem{rakhlin:13}
A.~Rakhlin and K.~Sridharan, ``{Optimization, learning, and games with
  predictable sequences},'' in \emph{NIPS}, 2013.

\bibitem{syrgkanis:15}
V.~Syrgkanis, A.~Agarwal, H.~Luo, and R.~Schapire, ``{Fast Convergence of
  Regularized Learning in Games},'' in \emph{Adv in Neural Information
  Processing Systems (NIPS)}, 2015.

\bibitem{bengio:06a}
Y.~Bengio, N.~L. Roux, P.~Vincent, O.~Delalleau, and P.~Marcotte, ``{Convex
  Neural Networks},'' in \emph{NIPS}, 2006.

\bibitem{bach:14}
F.~Bach, ``{Breaking the Curse of Dimensionality with Convex Neural
  Networks},'' in \emph{arXiv:1412.8690}, 2014.

\bibitem{aslan:13}
O.~Aslan, H.~Cheng, D.~Schuurmans, and X.~Zhang, ``{Convex Two-Layer
  Modeling},'' in \emph{NIPS}, 2013.

\bibitem{aslan:14}
O.~Aslan, X.~Zhang, and D.~Schuurmans, ``{Convex Deep Learning via Normalized
  Kernels},'' in \emph{NIPS}, 2014.

\bibitem{vonneumann:44}
J.~von Neumann and O.~Morgenstern, \emph{{Theory of Games and Economic
  Behavior}}.\hskip 1em plus 0.5em minus 0.4em\relax Princeton University
  Press, Princeton NJ, 1944.

\bibitem{parkes:15}
D.~C. Parkes and M.~P. Wellman, ``Economic reasoning and artificial
  intelligence,'' \emph{Science}, vol. 349, no. 6245, pp. 267--272, 2015.

\bibitem{bottou:91}
L.~Bottou and P.~Gallinari, ``A framework for the cooperation of learning
  algorithms,'' in \emph{NIPS}, 1991.

\bibitem{bb:12}
D.~Balduzzi and M.~Besserve, ``Towards a learning-theoretic analysis of
  spike-timing dependent plasticity,'' in \emph{Advances in Neural Information
  Processing Systems (NIPS)}, 2012.

\bibitem{balduzzi:13mv}
D.~Balduzzi, ``Randomized co-training: from cortical neurons to machine
  learning and back again,'' \emph{Randomized Methods for Machine Learning
  Workshop, Neural Inf Proc Systems (NIPS)}, 2013.

\bibitem{balduzzi:14cpm}
------, ``Cortical prediction markets,'' in \emph{Proc. 13th Int Conf on
  Autonomous Agents and Multiagent Systems (AAMAS)}, 2014.

\bibitem{bvb:15}
D.~Balduzzi, H.~Vanchinathan, and J.~Buhmann, ``Kickback cuts {B}ackprop's
  red-tape: {B}iologically plausible credit assignment in neural networks,'' in
  \emph{AAAI Conference on Artificial Intelligence (AAAI)}, 2015.

\bibitem{baldi:88}
P.~Baldi and K.~Hornik, ``Neural networks and principal component analysis:
  learning from examples without local minima,'' \emph{Neural Networks},
  vol.~2, no.~1, pp. 53--58, 1988.

\bibitem{saxe:14}
A.~M. Saxe, J.~L. McClelland, and S.~Ganguli, ``{Exact solutions to the
  nonlinear dynamics of learning in deep linear neural networks},'' in
  \emph{ICLR}, 2014.

\bibitem{foster:97}
D.~P. Foster and R.~V. Vohra, ``{Calibrated Learning and Correlated
  Equilibrium},'' \emph{Games and Economic Behavior}, vol.~21, pp. 40--55,
  1997.

\bibitem{aumann:74}
R.~J. Aumann, ``Subjectivity and correlation in randomized strategies,''
  \emph{J Math Econ}, vol.~1, pp. 67--96, 1974.

\bibitem{blum:07}
A.~Blum and Y.~Mansour, ``{From External to Internal Regret},'' \emph{JMLR},
  vol.~8, pp. 1307--1324, 2007.

\bibitem{goodfellow:13}
I.~J. Goodfellow, D.~Warde-Farley, M.~Mirza, A.~Courville, and Y.~Bengio,
  ``Maxout {N}etworks,'' in \emph{ICML}, 2013.

\bibitem{hazan:07}
E.~Hazan and S.~Kale, ``{Computational Equivalence of Fixed Points and No
  Regret Algorithms, and Convergence to Equilibria},'' in \emph{Adv in Neural
  Information Processing Systems (NIPS)}, 2007.

\bibitem{wan:13}
L.~Wan, M.~Zeiler, S.~Zhang, Y.~LeCun, and R.~Fergus, ``Regularization of
  {N}eural {N}etworks using {D}rop{C}onnect,'' in \emph{ICML}, 2013.

\bibitem{monderer:96}
D.~Monderer and L.~S. Shapley, ``{Potential Games},'' \emph{Games and Economic
  Behavior}, vol.~14, pp. 124--143, 1996.

\bibitem{hofbauer:02}
J.~Hofbauer and W.~H. Sandholm, ``{On the global convergence of stochastic
  fictitious play},'' \emph{Econometrica}, vol.~70, no.~6, pp. 2265--2294,
  2002.

\bibitem{hart:03}
S.~Hart and A.~Mas-Colell, ``Regret-based continuous-time dynamics,''
  \emph{Games and Economic Behavior}, vol.~45, pp. 375--394, 2003.

\bibitem{zinkevich:03}
M.~Zinkevich, ``Online {C}onvex {P}rogramming and {G}eneralized {I}nfinitesimal
  {G}radient {A}scent,'' in \emph{ICML}, 2003.

\bibitem{koren:13}
T.~Koren, ``{Open Problem: Fast Stochastic Exp-Concave Optimization},'' in
  \emph{COLT}, 2013.

\bibitem{baldi:12}
P.~Baldi and Z.~Lu, ``Complex-valued autoencoders,'' \emph{Neural Networks},
  vol.~33, pp. 136--147, 2012.

\bibitem{luo:16}
H.~Luo, A.~Agarwal, N.~Cesa-Bianchi, and J.~Langford, ``{Efficient Second Order
  Online Learning via Sketching},'' in \emph{arXiv:1602.02202}, 2016.

\bibitem{helmbold:00}
D.~P. Helmbold, N.~Littlestone, and P.~M. Long, ``{Apple Tasting},''
  \emph{Information and Computation}, vol. 161, pp. 85--139, 2000.

\bibitem{agarwal:14}
A.~Agarwal, D.~Hsu, S.~Kale, J.~Langford, L.~Li, and R.~Schapire, ``{Taming the
  Monster: A Fast and Simple Algorithm for Contextual Bandits},'' in
  \emph{ICML}, 2014.

\bibitem{williams:92}
R.~J. Williams, ``Simple {S}tatistical {G}radient-{F}ollowing {A}lgorithms for
  {C}onnectionist {R}einforcement {L}earning,'' \emph{Machine Learning},
  vol.~8, pp. 229--256, 1992.

\end{thebibliography}
\end{document}